\RequirePackage{fix-cm}

\makeatletter
\def\cl@chapter{}
\makeatother
\documentclass[smallcondensed]{svjour3}   
\smartqed 
\usepackage{latexsym}
\usepackage{hyperref}
\usepackage{algorithm, algorithmic}

\usepackage{amssymb, amsmath}
\usepackage{bbold}
\usepackage{shortcuts}
\usepackage{cleveref}
\usepackage[authoryear]{natbib}
\usepackage{float}

\usepackage{graphicx}
\usepackage{subfigure}
\journalname{myjournal}

\begin{document}

\title{Root-finding Approaches for Computing Conformal Prediction Set
}


\author{Eugene Ndiaye    \and
    Ichiro Takeuchi 
}


\institute{Eugene Ndiaye \at
H. Milton Stewart School of Industrial and Systems Engineering,\newline
Georgia Institute of Technology, Atlanta, GA, USA \\
     \email{endiaye3@gatech.edu}      
      \and
      Ichiro Takeuchi \at
      Faculty of Engineering, Nagoya University / RIKEN AIP, Furo-cho, Nagoya, Japan\\
      \email{ichiro.takeuchi@mae.nagoya-u.ac.jp}
      \and
      Correspondence to: [Eugene Ndiaye]
}

\date{Received: date / Accepted: date}

\maketitle

\begin{abstract}
Conformal prediction constructs a confidence set for an unobserved response of a feature vector based on previous identically distributed and exchangeable observations of responses and features. It has a coverage guarantee at any nominal level without additional assumptions on their distribution. Its computation deplorably requires a refitting procedure for all replacement candidates of the target response. In regression settings, this corresponds to an infinite number of model fits. Apart from relatively simple estimators that can be written as pieces of linear function of the response, efficiently computing such sets is difficult, and is still considered as an open problem. We exploit the fact that, \emph{often}, conformal prediction sets are intervals whose boundaries can be efficiently approximated by classical root-finding algorithms. We investigate how this approach can overcome many limitations of formerly used strategies; we discuss its complexity and drawbacks.\looseness=-1
\end{abstract}

\section{Introduction}
\label{sec:Introduction}
 
\citep{Gammerman_Vovk_Vapnik98, Vovk_Gammerman_Shafer05, Shafer_Vovk08} introduced Conformal Prediction (CP) as a general method for predicting a confidence set of a random variable from its point prediction. Given an observed data set $\Data_n = \{(x_1, y_1), \cdots, (x_n, y_{n})\}$ sampled from a distribution $\mathbb{P}$, it constructs a $100(1 - \alpha)\%$ confidence set that contains the unobserved response $y_{n+1}$ of a new instance $x_{n+1}$. In this way, it equips traditional statistical learning algorithms with a confidence value when predicting the response of a new test example.
The general idea is to learn a predictive model on the augmented database $\Data_{n+1}(z) = \Data_n \cup (x_{n+1}, z)$ where $z$ replaces the unknown response $y_{n+1}$. We can therefore define a prediction loss for each observation, and rank them. A candidate $z$ will be considered as conformal or typical if the rank of its loss is sufficiently small. The conformal prediction set will merely collect the most typical $z$ as a confidence set for $y_{n+1}$.
As long as the sequence $\{(x_i, y_i)\}_{i=1}^{n+1}$ is exchangeable \footnote{Their joint probability distribution is invariant \wrt permutation of the data.}, and the predictive model is invariant with respect to permutation of the data, this method benefits from a strong coverage guarantee without any assumption on the distribution. This holds for any finite sample size $n$. \\
 
Several extensions, and applications of conformal prediction have been developed for designing uncertainty sets in active learning \citep{Ho_Wechsler08}, anomaly detection \citep{Laxhammar_Falkman15, Bates_Candes_Lei_Romano_Sesia21}, image classification \citep{Angelopoulos_Bates_Malik_Jordan20}, few shot learning \citep{Fisch_Schuster_Jaakkola_Barzilay21}, time series \citep{Xu_Xie20} or to infer the performance guarantee for statistical learning algorithms \citep{Holland20, Cella_Martin20, Bates_Angelopoulos_Lei_Malik_Jordan21}. We refer to the extensive reviews in \citep{Balasubramanian_Ho_Vovk14} for other applications to artificial intelligence.\\
 
Despite these attractive properties, the computation of conformal prediction sets is challenging for regression problems since an infinite number of models must be fitted with an augmented training set $\Data_{n+1}(z)$, for all possible $z \in \bbR$. This is not only expensive, it is simply impossible in most cases. In general, efficiently computing conformal sets with the full data remains an open problem.
The current successful approaches for calculating the set of conformal predictions are twofold.
 
\begin{itemize}
\item Exhaustive search with a \emph{homotopy continuation}\footnote{\eug{also called numerical continuation or path following methods \citep{Allgower_Georg12} are techniques for numerical approximation of a solution curve implicitly defined by a system of equations. In our context, we are interested in solving optimization problems, and we need to describe the solution curve \wrt parameter changes of the first order optimality conditions.}}. The fundamental idea is to rely on the fact that the typicalness function that maps each candidate with the rank of its prediction loss is piecewise constant. As follows, if we carefully manage to list all its transition points, we can find exactly where it is above the prescribed confidence level.
For estimators that have a closed-form formula\eg Ridge \citep{Hoerl_62} or Lasso \citep{Tibshirani96}, it is possible to draw the solutions curve \wrt the input candidate $z$. They are often pieces of linear function which enable the exhaustive listing of the change points of the rank function; see \citep{Nouretdinov_Melluish_Vovk01}, and \citep{Lei19}.\\
 
\item \emph{Inductive confidence machine} also called \emph{Splitting} \citep{Papadopoulos_Proedrou_Vovk_Gammerman02, Lei_GSell_Rinaldo_Tibshirani_Wasserman18}. The observed dataset is divided into two parts. A proper training set to fit the regression model, and an independent calibration set to calculate prediction losses, and ranks. This method is the most computationally efficient because it requires only a single model fit on a sub-part of the data. Separating the roles of the data to build the model, and to evaluate its performance avoids refitting without loss of coverage guarantees. The use of splitting techniques in statistics can be dated at least to \citep{Cox75}
\end{itemize}
 
These strategies have some noticeable limitations. The homotopy methods rely on strong assumptions on the model fit, and are numerically unstable due to multiple matrix inversions that are potentially poorly conditioned. They can suffer from exponential complexity in the worst cases, and must frequently be abandoned because of extremely small step sizes \citep{Gartner_Jaggi_Maria12, Mairal_Yu12}. The data splitting approach does not use all the data in the training phase. It generally results in a wider confidence region\ie of wider size.
As an alternative, a common heuristic unduly restricts the function evaluations to an arbitrary discrete grid of trial values $z$, and select the most typical one among them. These strategies might lose the coverage guarantee, and are still computationally inefficient. As a viable alternative, one relaxes the exact computation of the regression model at every step, and then approximately follows the homotopy continuation path by tightly controlling the optimization error as in \citep{Giesen_Jaggi_Laue10, Ndiaye_Le_Fercoq_Salmon_Takeuchi2019}. \citep{Ndiaye_Takeuchi19} has shown this is a safer discretization strategy, and that it can cope with more general nonlinear regressions. Still, it is so far limited to convex problems with strong regularity assumptions on the model fit, and fails to be applicable to most machine learning prediction methods.
 
\subsection*{Summary of the contributions}
 
We build on the striking remark that for common practical situations, the conformal prediction set is a bounded interval of the real line. Its boundaries are the roots of coverage level $\alpha$ minus the typicalness function, and these can be efficiently computed by a root-finding algorithm such as bisection search, with high precision, and without suffering from the limitations mentioned above. Despite its simplicity, it overcomes the limitations of the aforementioned strategies, and significantly improves, and extends the applicability of full conformal prediction to problems where it was considered intractable so far.\\
 
We highlight some advantages of our approach.
 
\begin{itemize}
 
\item \textbf{\emph{Efficiency}}: we demonstrate that computing a full conformal prediction set is tractable under mild assumptions. Relying on a bisection search, approximating the boundaries of the full exact conformal set at a prescribed accuracy $\epsilon > 0$, requires about $O(\log_2({1}/{\epsilon}))$ number of model fit. The latter, trained on the whole data, allows to obtain a more informative confidence set than splitting methods. Accordingly, we maintain both statistical, and computational efficiency. \looseness=-1\\
 
\item \textbf{\emph{Flexibility}}: our strategy offers considerable freedom on the choice of the regression estimator. For example, it can be defined as an output of a gradient descent process that maximizes a likelihood. It can be terminated when the norm of the gradient is smaller than a tolerance $\epsilon_0$ or after $100$ iterations of the algorithm. Consequently, the estimator can be parameterized by the number of iterations or the optimization error resulting from an iterative process as long as the symmetry of the data is preserved. The proposed root-finding approaches are easily applicable to more sophisticated recent machine learning techniques, such as deep neural networks or models involving a non-convex regularization. \looseness=-1\\
 
\item \textbf{\emph{Simplicity}}: the proposed methods are straightforward to implement. One substantially benefits from freely available scientific computing software packages like \texttt{scikit-learn} \citep{Pedregosa_etal11} or \texttt{scipy} \citep{Virtanen20} to adjust models, and find the endpoints of the conformal set. \looseness=-1
\end{itemize}
 
We also introduce an interpolation point of view of grid based approaches that properly justifies how the coverage guarantee can be maintained along with reduced computational time. In the case where a piecewise linear (or constant) interpolation scheme, and a simple conformity score (for example\eug the absolute value) is used, the assumption that the conformal set is an interval is not required. The computations can be easily carried following a homotopy strategy.
To further reduce the number of model evaluations, we additionally provide a differentiable approximation of the rank function which effectively improves the computational efficiency of the root-finding solvers. We carefully analyze its coverage guarantee, and point out the trade-off between calibration, and number of model evaluations when such smoothing techniques are used. Such smoothing is mainly beneficial when a high precision is required.

\paragraph{Code availability.} The source of our implementation is available at
\begin{center}
\url{https://github.com/EugeneNdiaye/rootCP}
\end{center}
 
\paragraph{Notation.}
For a non zero integer $n$, we denote $[n]$ the set $\{1, \cdots, n\}$. We denote by $Q_{1 - \alpha}$, the $(1 - \alpha)$-quantile of a real valued sequence $(U_i)_{i \in [n + 1]}$, defined as the variable $Q_{1 - \alpha} = U_{(\lceil (n+1)(1-\alpha) \rceil)}$, where $U_{(i)}$ are the $i$-th order statistics. The interval $[a - \tau, a + \tau]$ will be denoted $[a \pm \tau]$. For an index $j$ in $[n+1]$, the rank of $U_j$ among $U_1, \cdots, U_{n+1}$ is defined as $\mathrm{Rank}(U_j) = \sum_{i=1}^{n+1}\mathbb{1}_{U_i \leq U_j}$.

\section{Conformal Prediction}

We recall the arguments presented in \citep{Vovk_Gammerman_Shafer05, Shafer_Vovk08, Lei_GSell_Rinaldo_Tibshirani_Wasserman18} while underlining in details the intuitions, and principles that sustain the construction, and validity of conformal prediction.
Let us consider an input random variable $X$, and output $Y$. The goal is to construct a confidence set for the variable $Y$\ie find a set $\mathcal{C}(X)$ such that
\begin{equation}\label{eq:def_confidence_set}
\mathbb{P}(Y \in \mathcal{C}(X)) \geq 1 - \alpha, \quad \forall \alpha \in (0, 1) \enspace.
\end{equation}
Given a prediction function $\mu(\cdot)$ that maps the input to the output space, and a loss measure $S$, one can assess the prediction error as $E = S(Y, \mu(X))$. It is a random variable with cumulative distribution function $F$ and quantile $Q$ defined as $F(z) = \mathbb{P}(E \leq z)$ and
$Q(\delta) = \inf\{z \in \bbR:\, F(z) \geq \delta \}$.
The main tool for building a set $\mathcal{C}(X)$ that satisfies the probabilistic guarantee in \Cref{eq:def_confidence_set}, is the following classical result\footnote{By definition, we have $Q(\delta)$ is the smallest real value $z$ such that $F(z) \geq \delta$. Thus $\delta \leq F(Q(\delta)) = \mathbb{P}(E \leq Q(\delta)) = \mathbb{P}( F(E) \leq \delta)$.}:
\begin{equation}\label{eq:fundamental_simulation_lemma}
\forall \delta \in (0, 1), \qquad \mathbb{P}(F(E) \leq \delta) \geq \delta \enspace.
\end{equation}
It implies $F(E) = F(S(Y, \mu(X)) \leq 1 - \alpha$ with probability larger than or equal to the confidence level $\delta = 1 - \alpha$. One then defines a confidence set for $Y$ as the collection of candidate $z$ that satisfy the same inequality\ie
$$\mathcal{C}(X) = \{z : F(S(z, \mu(X)) \leq 1 - \alpha\} \enspace.$$

It turns out that the same principle can be applied to compute a confidence set for sequential observations. To do so, the coverage bound in \Cref{eq:fundamental_simulation_lemma} can be extended to empirical cumulative distribution and empirical quantile functions defined as:\looseness=-1
\begin{align*}
&F_{n+1}(z) = \frac{1}{n+1}\sum_{i=1}^{n+1} \mathbb{1}_{E_i \leq z},
&Q_{n+1}(\delta) = \inf\{z \in \bbR: F_{n+1}(z) \geq \delta\} \enspace.
\end{align*}
\begin{lemma}\label{lm:empirical_repartition_lemma}
For a sequence of exchangeable random variables $E_1, \cdots, E_{n+1}$, it holds
$\mathbb{P}(F_{n+1}(E_{n+1}) \leq \delta) \geq \delta$, for any $\delta \in (0,1)$.
\end{lemma}

\begin{proof}  We follow the proof in \citep{Romano_Patterson_Candes19}.
By definition of the empirical quantile, we have $$\delta \leq F_{n+1}(Q_{n+1}(\delta)) = \frac{1}{n+1}\sum_{i=1}^{n+1} \mathbb{1}_{E_i \leq Q_{n+1}(\delta)} \enspace.$$
Taking the expectation on both side, we have
\begin{align*}
\delta &\leq \frac{1}{n+1}\sum_{i=1}^{n+1} \mathbb{P}(E_i \leq Q_{n+1}(\delta))
= \frac{1}{n+1}\sum_{i=1}^{n+1} \mathbb{P}(F_{n+1}(E_i) \leq \delta) \enspace.
\end{align*}

Moreover, we have for any $i$ in $[n]$, $\mathbb{P}(F_{n+1}(E_i) \leq \delta) = \mathbb{P}(F_{n+1}(E_{n+1}) \leq \delta)$ by excheangeability. Hence the result.
\begin{flushright}
    $\blacksquare$
\end{flushright}
\end{proof}

Using \Cref{lm:empirical_repartition_lemma}, we have $F_{n+1}(E_{n+1}) \leq 1 - \alpha$ with probability larger than or equal to the confidence level $\delta = 1 - \alpha$. Given the $n$ previous observations, one can define a confidence set for an unobserved variable $E_{n+1}$ as the random set $$\{z : F_{n+1}(z) \leq 1 - \alpha\} \enspace.$$

In supervised statistical learning problems, where we observe both the responses, and the features, one can apply this principle while taking benefits of an underlying model trained on the observed data. For the augmented dataset $\mathcal{D}_{n+1}(z) = \mathcal{D}_{n} \cup \{(x_{n+1}, z)\}$ for $z \in \bbR$, an example of predictive model is given as
$\mu_z(x) = \Phi(x, \hat \beta(z))$, where $\Phi$ is a regression model\eg a kernel machine or a Deep Neural Network with parameter $\hat\beta(z)$ adjusted on the data. For example, by using empirical risk minimization principle, one defines
\begin{equation}\label{eq:model_optimization}
\hat \beta(z) \in \argmin_{\beta \in \bbR^p} L(\beta \mid \mathcal{D}_{n+1}(z)) + \lambda \Omega(\beta) \enspace,
\end{equation}
where $\lambda > 0$ and
$L(\beta \mid \mathcal{D}_{n+1}(z)) = \sum_{i=1}^{n} \ell(y_i, \Phi(x_i, \beta)) + \ell(z, \Phi(x_{n+1}, \beta))$ is the data fitting term and the regularization function $\Omega$ enforces structured solutions\eg sparsity.

\paragraph{Examples.}
A popular example of an instance-wise loss function found in the literature is the \textit{power norm}, where $\ell(a, b) = |a - b|^q$. When $q=2$, this corresponds to classical linear regression. Cases where $q \in (0, 2)$ are common in robust statistics. In particular, $q=1$ is known as least absolute deviation. The \textit{logcosh} loss $\ell(a, b) = \gamma\log(\cosh(a - b)/\gamma)$ is a differentiable alternative to the $\ell_{\infty}$-norm. One can also have the \textit{Linex} loss function \citep{Gruber10, Chang_Hung07} which provides an asymmetric loss $\ell(a, b) = \exp(\gamma(a - b)) - \gamma(a - b) - 1$, for $\gamma \neq 0$. The regularization functions $\Omega$\eg Ridge \citep{Hoerl_Kennard70} or sparsity inducing norms \citep{Bach_Jenatton_Mairal_Obozinski12, Obozinski_Bach16} can be considered as well as non convex penalties  \citep{Xie_Huang09}. \\

Given the fitted model $\mu_z(\cdot)$ and a loss measure $S$, let us define the sequence of instance-wise prediction errors as:
\begin{align*}
\forall i \in [n],\, E_{i}(z) = S(y_i,\, \mu_z(x_{i})) \text{, and } 
E_{n+1}(z) = S(z,\, \mu_z(x_{n+1})) \enspace.
\end{align*}
The sequence $\{E_{1}(y_{n+1}), \cdots, E_{n}(y_{n+1}),E_{n+1}(y_{n+1})\}$ is exchangeable as long as the data $\{(x_i, y_i)\}_{i=1}^{n+1}$ is exchangeable, and the model fit $\mu_z(\cdot)$ is invariant \wrt permutation of the data. We can then apply \Cref{lm:empirical_repartition_lemma} to obtain a coverage guarantee.\looseness=-1
\begin{definition} The full conformal prediction set is formally defined as
\begin{equation}\label{eq:conformal_estimate_set}
\Gamma^{(\alpha)}(x_{n+1}) = \{z : F_{n+1}(E_{n+1}(z)) \leq 1 - \alpha\} \enspace,
\end{equation}
where
\begin{equation}\label{eq:conformal_cdf}
F_{n+1}(E_{n+1}(z)) = \frac{1}{n+1}\sum_{i=1}^{n+1} \mathbb{1}_{E_i(z) \leq E_{n+1}(z)} \enspace.
\end{equation}
The term "full" refers to the fact that the entire data set is used to fit the regression model; in contrast to the splitting approach presented in detail below. 
\end{definition}
\Cref{lm:empirical_repartition_lemma} implies that the set $\Gamma^{(\alpha)}(x_{n+1})$ is a valid confidence set for $y_{n+1}$ in the sense of \Cref{eq:def_confidence_set}\ie
$\mathbb{P}(y_{n+1} \in \Gamma^{(\alpha)}(x_{n+1})) \geq 1 - \alpha$ for any $\alpha$ in $(0, 1)$.
Somehow, the refitting procedure with the extended dataset $\mathcal{D}_{n+1}(z)$, puts all the variables on equal feet, and preserves the exchangeability of the sequence of prediction errors.
Using the rank function, we have $(n+1) F_{n+1}(E_{n+1}(z)) = \mathrm{Rank}(E_{n+1}(z))$,
and one can rewrite the CP set as (which is in fact the traditional notation) \looseness=-1
\begin{equation*}
\Gamma^{(\alpha)}(x_{n+1}) = \{z : \pi(z) \geq \alpha\} \enspace,
\end{equation*}
where $z \mapsto \pi(z)$ is the typicalness function that measures how conformal a candidate is. It is defined as
\begin{equation}\label{eq:typicalness}
\pi(z) = 1 - \frac{1}{n+1} \mathrm{Rank}(E_{n+1}(z)) 
= 1 - F_{n+1}(E_{n+1}(z)) \enspace.
\end{equation}
\Cref{lm:empirical_repartition_lemma} reads $\mathbb{P}(\pi(y_{n+1}) \leq \alpha) \leq \alpha$\ie the random variable $\pi(y_{n+1})$ takes small values with small probability. Thus, it is unlikely that $y_{n+1}$ will take the value $z$ when $\pi(z)$ is small. More precisely, $\pi(y_{n+1})$ is (sub) uniformly distributed as usual for classical statistics for hypothesis testing. For example $p$-value function satisfies such a property under the null hypothesis; see \citep[Lemma 3.3.1]{Lehmann_Romano06}. One can then interpret the typicalness $\pi(\cdot)$ as a $p$-value function for testing the null hypothesis $H_0: y_{n+1}=z$ against the alternative $H_1: y_{n+1} \neq z$, for $z$ in $\bbR$. The conformal prediction set merely corresponds to the collection of candidate $z$ for which the null hypothesis $H_0$ is not rejected.

\section{Computing Conformal Prediction Set}
 
For regression problem where $y_{n+1}$ lies in a subset of $\bbR$, one need to evaluate $\pi(z)$ in \Cref{eq:typicalness}, and so refitting the model $\mu_z(\cdot)$ for infinitely many candidate $z$. This merely renders the overall computation challenging, and leaves the problem open in general.
Nevertheless, some peculiar regularity structure of the typicalness function $\pi(\cdot)$ can be exploited. For example, by utilizing the fact that it is piecewise constant, it is sufficient to enumerate the transition points (when they are finite) to compute the conformal set. This is possible for a limited number of cases\eg Ridge or Lasso where the map $z \mapsto \mu_z(\cdot)$ can be explicitly described. Unfortunately, only a very small class of statistical learning problems has such nice regularity structure. For other estimators, the computation of CP set when $y_{n+1}$ can take countless number of values, is unclear.
 
\paragraph{Splitting.}
To overcome this issue, the split conformal prediction set introduced in \citep{Papadopoulos_Proedrou_Vovk_Gammerman02}, separates the model fitting, and the score ranking step. Let us define
\begin{itemize}
\item the training set $\mathcal{D}_{\rm{tr}} = \{(x_1, y_1), \cdots, (x_m, y_m)\}$ with $m < n$,
\item the calibration set $\mathcal{D}_{\rm{cal}} = \{(x_{m+1}, y_{m+1}), \cdots, (x_n, y_n)\}$.
\end{itemize}
Then the model is fitted on the training set $\mathcal{D}_{\rm{tr}}$ to get $\mu_{\rm{tr}}(\cdot)$, and define the score function on the calibration set $\mathcal{D}_{\rm{cal}}$:
\begin{align*}
\forall i \in [m+1, n], \, E_{i}^{\rm{cal}} = S(y_i,\, \mu_{\rm{tr}}(x_i)), \text{ and }
E_{n+1}^{\rm{cal}}(z) = S(z,\, \mu_{\rm{tr}}(x_{n+1})) \enspace.
\end{align*}
Thus, we obtain the split typicalness function as
\begin{align*}
\pi_{\rm{split}}(z) &= 1 - F_{\rm{split}}(E_{n+1}^{\rm{cal}}(z)), \text{ where}\\
F_{\rm{split}}(E_{n+1}^{\rm{cal}}(z)) &= \frac{1}{n - m + 1}\sum_{i=m+1}^{n+1} \mathbb{1}_{E_{i}^{\rm{cal}} \leq E_{n+1}^{\rm{cal}}(z)} \enspace.
\end{align*}
The latter is proportional to the rank of the $(n+1)$th score on the calibration set. Finally, we define
\begin{align*}
\Gamma_{\rm{split}}^{(\alpha)}(x_{n+1}) &= \{z: \pi_{\rm{split}}(z) \geq \alpha \}
= \{z: E_{n+1}^{\rm{cal}}(z) \leq Q_{1-\alpha}^{\rm{cal}}\}
\enspace,
\end{align*}
where $Q_{1-\alpha}^{\rm{cal}}$ is the $(1-\alpha)$ quantile of the calibration scores $\{E_{m+1}^{\rm{cal}}, \cdots, E_{n+1}^{\rm{cal}}\}$.
When the score function is the absolute value $S(a, b) = |a - b|$, the split CP set is the interval $\Gamma_{\rm{split}}^{(\alpha)}(x_{n+1}) = [\mu_{\rm{tr}}(x_{n+1}) \pm Q_{1-\alpha}^{\rm{cal}}]$.
While this approach avoids the computational bottleneck, the statistical efficiency of the model can be reduced due to a significantly smaller sample available during the training, and calibration phase. Moreover, the length of the split conformal set tends to have a higher variance. \eug{In general, the proportion of training vs calibration set is a hyperparameter that requires appropriate tuning: a small calibration set leads to highly variable conformity scores, and a small training set leads to poor model fitting. In all our experiments, we set the splitting proportion to $2$, which means that the two sets play symmetric roles. Since the sequence of scores $\{ E_{m+1}^{\rm{cal}}, \ldots, E_{n}^{\rm{cal}}, E_{n+1}^{\rm{cal}}(z)\}$ is exchangeable, the \Cref{lm:empirical_repartition_lemma} implies that $\mathbb{P}(y_{n+1} \in \Gamma_{\rm{split}}^{(\alpha)}(x_{n+1})) \geq 1 - \alpha$.}
 
\eug{\paragraph{Cross-conformal Predictors.} The trade-off mentioned above is very recurrent in machine learning, and often appears in the debate between bias reduction, and variance reduction. It is often decided by the cross-validation method with several folds \citep{Arlot_Celisse10}. \textit{Cross-conformal predictors} \citep{Vovk15} follow the same ideas, and exploit the full dataset for calibration, and significant proportions for training the model. The dataset is partitioned into $K$ folds, and one performs a split conformal set by sequentially defining the $k$th fold as calibration set, and the remaining as training set for $k \in \{1, \ldots, K\}$. However, aggregating the different pi-values is not straightforward, and the validity of the method might be jeopardized without stronger assumptions on the score function, see \citep{Carlsson_Eklund_Norinder14, Linusson_Norinder_Bostrom_Johansson_Lofstrom17}. More precisely, it can be shown that the confidence level is inflated by a factor of $2$\ie the (not improvable) theoretical coverage level is $1 - 2\alpha$ instead of $1 - \alpha$, see \citep{Barber_Candes_Ramdas_Tibshirani21}. Under additional stability assumption, Cross-conformal predictors can only \textit{approximately} achieve the target coverage $1 - \alpha$. Otherwise, without approximation, in order to remove the factor $2$, one can consider an overly conservative set whose extremity are defined as the smallest, and largest residual over all leave-one-out residuals. The leave-one-out (also called Jackknife) CP set, will require $K=n$ model fit which is prohibitive even when $n$ is moderately large. On the other hand, the $K$-fold version will require $K$ model fit but will come at the cost of fitting on a lower sample size, and will leads to an additional excess coverage of $O(\sqrt{2/n})$. A Bootstrap version \citep[Appendix B]{Vovk15} will suffer from the same inflation \citep{Kim_Xu_Barber20}. In all cases, we are not aware of a (variant of) cross-conformal predictors that simultaneously achieves $1 - \alpha$ provable coverage guarantee, and a non-conservative prediction set. Nevertheless, the practical performance is fairly acceptable both computationally, and statistically. In this paper, we only compare with the methods that provably achieve the prescribed $1 - \alpha$ confidence level, namely the splitting method, and the oracle conformal prediction described in \Cref{sec:experiments}.
}
 
\subsection{Approximation to a Prescribed Accuracy}
\label{subsec:Approximation_to_a_Prescribed_Accuracy}
 
In this paper, we promptly take advantage of the remarkable fact that the conformal regions are \emph{often} intervals. We subsequently take an alternative direction which carefully avoids tracking the integral path of all model fit, and also avoids any data splitting. When the $(1-\alpha)$-level set of the function $z \mapsto F_{n+1}(E_{n+1}(z))$ is convex\eg \Cref{fig:initialization}, we propose to employ a numerical root-finding solver to approximate the endpoints of the interval. \eug{The statistical validity is automatic, we obtain simultaneously an upper, and lower bound on each extremity of the confidence set, and the approximation error $\epsilon$ can be made arbitrarily small at the cost of $O(\log(1/\epsilon))$ number of model fit; without inflation of the confidence level.\looseness=-1}
 
\subsection*{Outline of the Algorithm: \texttt{rootCP}}
 
Assuming that the conformal set is a non empty interval of finite length, we denote
$$\Gamma^{(\alpha)}(x_{n+1}) = [\ell_{\alpha}(x_{n+1}), u_{\alpha}(x_{n+1})] \enspace.$$
Given a tolerance $\epsilon > 0$, we proceed as follows:
\begin{enumerate}
\item find $z_{\min} < z_0 < z_{\max}$ such that
\begin{equation}\label{eq:initialization_condition}
\pi(z_{\min}) < \alpha < \pi(z_{0}) \text{ and } \alpha > \pi(z_{\max}) \enspace.
\end{equation}
\item Perform a bisection search in $[z_{\min}, z_0]$. It will output a point $\hat \ell$ such that $\ell_{\alpha}(x_{n+1})$ belongs to $[\hat \ell \pm \epsilon]$ after at most $\log_2(\frac{z_0 - z_{\min}}{\epsilon})$ iterations.
 
\item Perform a bisection search in $[z_0, z_{\max}]$. It will output a point $\hat u$ such that $u_{\alpha}(x_{n+1})$ belongs to $[\hat u \pm \epsilon]$ after at most $\log_2(\frac{z_{\max} - z_0}{\epsilon})$ iterations.
\end{enumerate}
 
\begin{figure*}
\centering
\subfigure[Failed initialization at first trial]{\includegraphics[width=0.49\columnwidth]{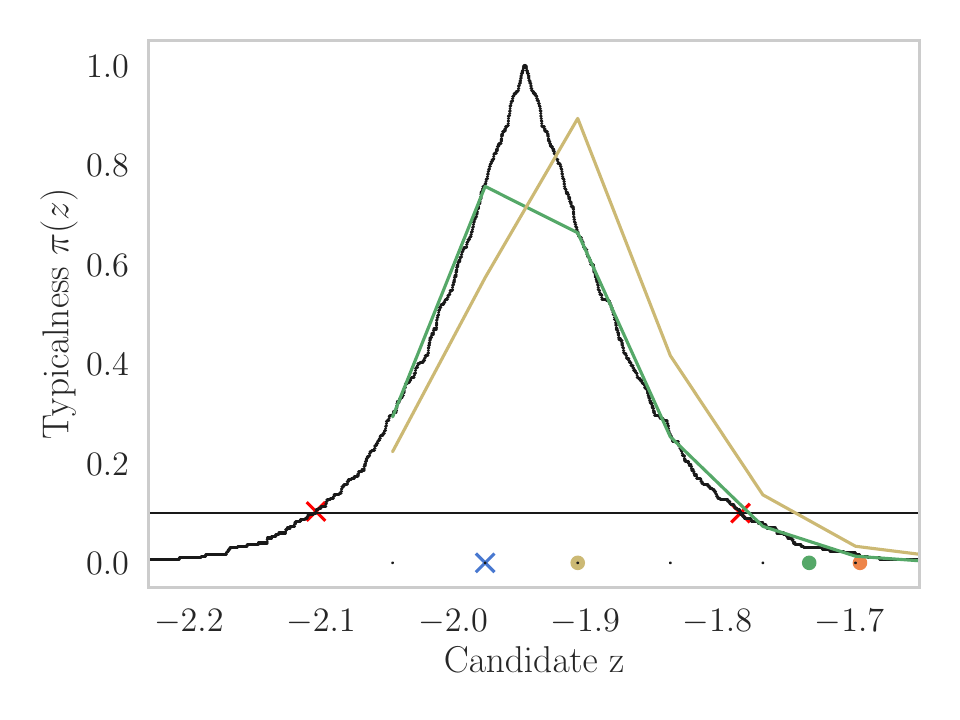}}
\label{subfig:initialization_fail}
%
\subfigure[Successful initialization at first trial]{\includegraphics[width=0.49\columnwidth]{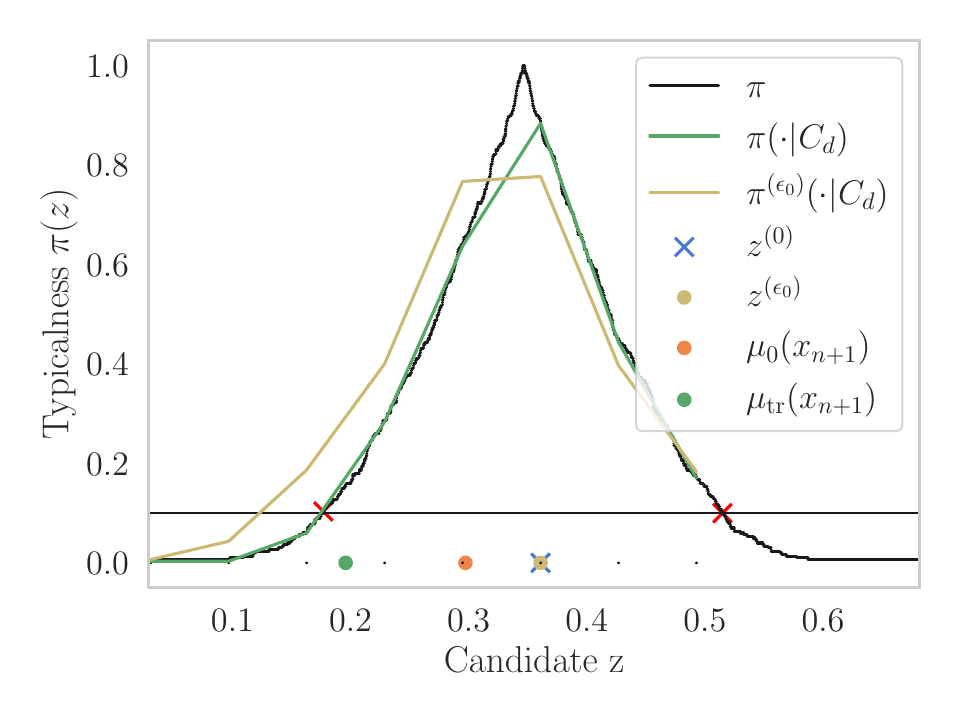}}
\label{subfig:initialization_succeed}
\caption{\textbf{Illustration of the initialization steps} when both the initial prediction based on observed data $z_0 = \mu_{\mathcal{D}_n}(x_{n+1})$, and midpoint of split conformal interval $\mu_{\rm{tr}}(x_{n+1})$ fails to be in the conformal prediction set whose boundaries are delimited by the red crosses. The synthetic data are generated with $\texttt{sklearn}$ as $X, y =$ \texttt{make\_regression}$(n=300, p=50)$. We choose an optimization accuracy of $\epsilon_0=\norm{(y_1, \cdots, y_n)}_{2}^{2}/10^4$ for approximating the ridge estimator. The trial points are $C_d = \{z_1, \ldots, z_d\}$ with $d = 10$, and we denoted $\displaystyle z^{(\epsilon)} \in \argmax_{z \in C_d}\pi^{(\epsilon)}(z)$ is the most conformal trial candidate at precision $\epsilon \geq 0$. To be more explicit, the approximated conformity functions obtained from the $C_d$ grid are denoted $\pi^{(\epsilon_0)}(\cdot \mid C_d)$ when early stopping at optimization accuracy $\epsilon_0$ is used, and $\pi(\cdot \mid C_d)$ when the exact solution is used. \label{fig:initialization}}
\end{figure*}
 
\begin{figure*}
\centering
\subfigure[Boundaries of conformal sets]{\includegraphics[width=0.49\columnwidth]{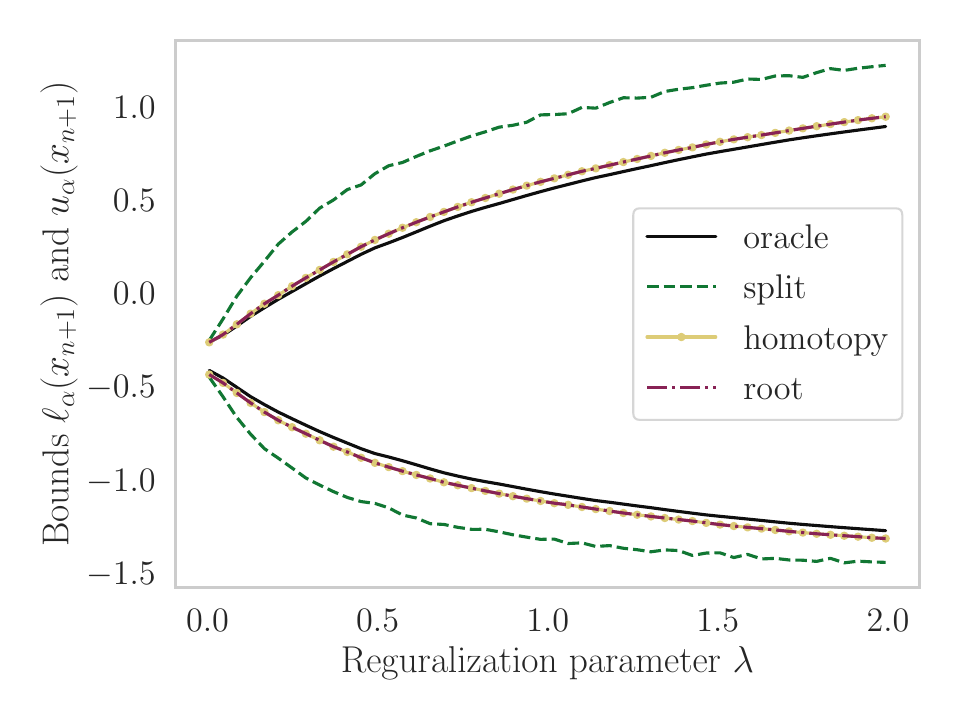}}
\label{subfig:ridge_path_cp}
%
\subfigure[Computational times]{\includegraphics[width=0.49\columnwidth]{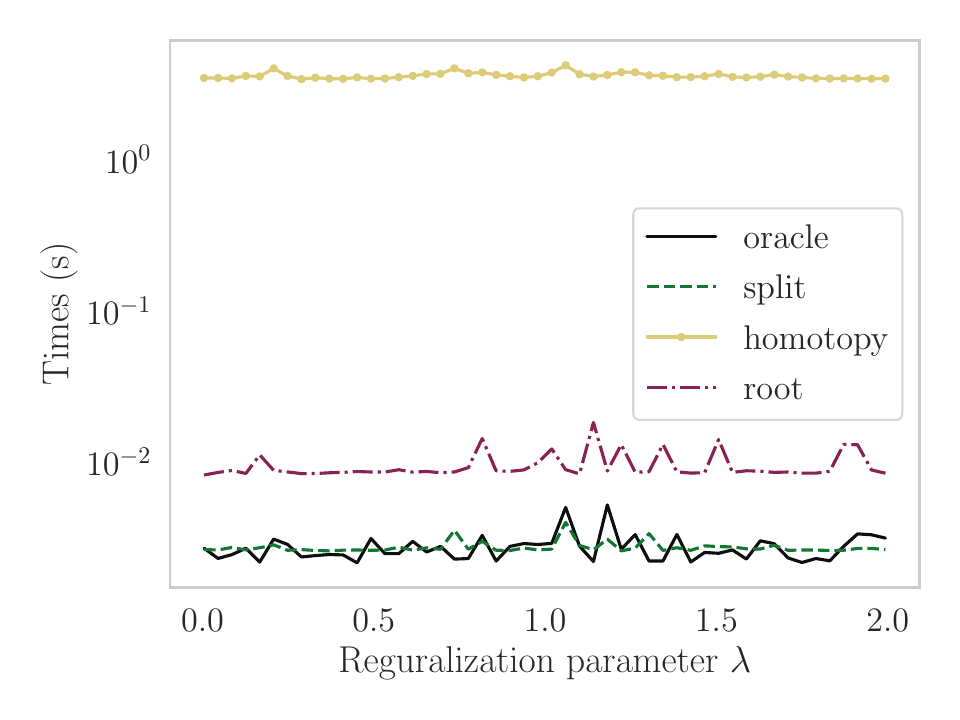}}
\label{subfig:ridge_path_time}
\caption{\textbf{Benchmark on ridge regression.} Conformal prediction set computed with various regularization parameter on synthetic dataset generated from $\texttt{sklearn}$ as $X, y =$ \texttt{make\_regression}$(n=1000, p=100)$ with $90$ informative features. For the splitting method, we average the results of $100$ independent run. For the proposed root finding method, we approximate the boundaries of the exact set at precision $10^{-12}$. \label{fig:benchmark_ridge}}
\end{figure*}
 
\subsection*{Initialization}
%
For the initial lower, and upper bounds, we suggest
\begin{equation*}
z_{\min} = \min_{i \in [n]} y_i \text{ and } z_{\max} = \max_{i \in [n]} y_i \enspace.
\end{equation*}
For most of the situations encountered in our numerical experiments, we consistently get $\pi(z_{\min})$, and $\pi(z_{\max})$ both smaller than the threshold level $\alpha$. Otherwise, we can always take values even farther apart without affecting the complexity thanks to the logarithmic dependence in the length of the initialization brackets. This is especially necessary when the total number of samples $n$ is small\footnote{Indeed, we have
$\mathbb{P}(y_{n+1} \in [z_{\min}, z_{\max}]) \geq 1 - \frac{1}{n+1}$. Hence when $n$ is sufficiently large\ie $n \geq 1 + 1/\alpha$, then $[z_{\min}, z_{\max}]$ is a $(1-\alpha)$ confidence set.}.
The most crucial part is to choose $z_0$ so that $\pi(z_0) > \alpha$. It is equivalent to get a point in the interior of the conformal set itself. In the ideal case where the length of the conformal set is extremely small, finding an initialization point might be notoriously hard. Indeed, it corresponds to a \emph{rare event} equivalent to sampling a point in a low probability region.
We adopted a simple strategy which consists in estimating $y_{n+1}$ with the observed data $\mathcal{D}_n$. We subsequently denote it
\begin{equation*}
z_0 = \mu_{\mathcal{D}_n}(x_{n+1}) \enspace.
\end{equation*}
In our sequence of repetitive numerical experiments, this choice rarely fails.
Naturally, its success depends on the prediction capabilities of the model fit. In the rare cases where it fails, we propose to test the initialization condition on some query points selected on an initial estimation $[z_{\alpha}^{-}, z_{\alpha}^{+}]$ of the CP set. This localization step aims to exploit additional problem structure, and can be interpreted as an iterative importance sampling to maintain a reasonably low computational cost.
 
\begin{enumerate}
\item \emph{Localization.} Given an easy to compute estimate set $[z_{\alpha}^{-}, z_{\alpha}^{+}]$ that is potentially larger\footnote{One can slightly enlarge the set by taking $[z_{\alpha}^{-} - 0.5 (z_{\alpha}^{+} - z_{\alpha}^{-}), z_{\alpha}^{+} + 0.5 (z_{\alpha}^{+} - z_{\alpha}^{-})].$} than the targeted conformal set, we select its mid point
$$ z_0 = \frac{z_{\alpha}^{+} + z_{\alpha}^{-}}{2} \enspace.$$
If $z_0$ satisfies $\pi(z_0) > \alpha$, then we have a valid initialization by paying only a single model fit. Otherwise, we run the next step on the bracket $[z_{\alpha}^{-}, z_{\alpha}^{+}]$.
 
For instance, one can use the interval obtained from the splitting approach $[z_{\alpha}^{-}, z_{\alpha}^{+}] = \Gamma_{\rm{split}}^{(\alpha)}(x_{n+1})$ or a rough approximation $[z_{\alpha}^{-}, z_{\alpha}^{+}] = \{z: \pi_{\mathcal{D}_n}(z) > \alpha\}$ where $\pi_{\mathcal{D}_n}(\cdot)$ is an unsafe estimation of $\pi(\cdot)$ with $\mu_z(x)$ replaced by $\mu_{\mathcal{D}_n}(x)$ for any candidate $z$, and any input feature $x$.\looseness=-1 \\
 
\item \emph{Sampling.} For a small number $d$\eg $d=5$, and given a bracket search $[z^-, z^+]$, select candidates $C_d = \{z_1, \cdots, z_d\}$ uniformly. If there is $\pi(z_0) > \alpha$ for a $z_0$ in $C_d$, we have a valid initialization. Otherwise, we use these query points to interpolate the model fit as in \cref{eq:piece_linear_interpolation}. Thus, by selecting additional points that have a higher typicalness according to the interpolated model, one can refine the sampling set $C_d$, and repeat the process.\\
 
For completeness, we summarize the procedure in \Cref{alg:rootCP}.
 
\end{enumerate}
 
\begin{algorithm}
    \caption{\texttt{rootCP} using bisection search \Cref{alg:bisection}}
    \label{alg:rootCP}
 \begin{algorithmic}
    \STATE {\bfseries Input:} data $\Data_n = \{(x_1, y_1), \ldots, (x_n, y_n)\}$, and $x_{n+1}$
    \STATE Coverage level $\alpha \in (0, 1)$, accuracy $\epsilon > 0$
    \STATE {\bfseries Output:} approximation of $\Gamma_{\rm{up}}^{(\alpha)}(x_{n+1}) = [\ell_{\alpha}(x_{n+1}), u_{\alpha}(x_{n+1})]$ up to accuracy $\epsilon$.
    \STATE
    \STATE \quad \# \emph{Initialization}
    \STATE Set $z_{\min} = y_{(1)} \text{ and } z_{\max} = y_{(n)}$
    \STATE Set $z_0 = \mu_{0}(x_{n+1})$ where we fitted a model $\mu_{0}$ on the training data $\Data_n$
    \IF{$z_{\min} < z_0 < z_{\max}$ such that
    \begin{equation}\label{eq:initialization_success}
        \pi(z_{\min}) < \alpha < \pi(z_{0}) \text{ and } \alpha > \pi(z_{\max})
    \end{equation}}
        \STATE
        \STATE\# \emph{Approximation of Boundary Points}
        \begin{enumerate}
            \item $\hat \ell = \texttt{bisection\_search}(\pi - \alpha, z_{\min}, z_0, \epsilon)$    
            \item $\hat u = \texttt{bisection\_search}(\pi - \alpha, z_0, z_{\max}, \epsilon)$
        \end{enumerate}
        \STATE {\bfseries Return:} $[\hat \ell, \hat u]$
    \ELSE
        \STATE
        \STATE\# \emph{Initialization Failure}
        \STATE Perform a heuristic search in a grid point $C_d = \{z_1, \ldots, z_d\}$ until \Cref{eq:initialization_success} is met
        \STATE Go to root-finding steps $1.$ and $2.$
    \ENDIF
    \STATE {\bfseries Return:} \emph{Global Failure}: refer to alternative methods (\egc \texttt{splitCP} or \texttt{interpolCP})
 \end{algorithmic}
 \end{algorithm}
 
 \begin{algorithm}
    \caption{\texttt{bisection\_search}}
    \label{alg:bisection}
 \begin{algorithmic}
    \STATE {\bfseries Input:} Function $f$, scalars $a < b$ such that $\sign(f(a)) \neq \sign(f(b))$, accuracy $\epsilon > 0$
    \STATE Set $\texttt{max\_iter} = \log_2(\frac{b - a}{\epsilon})$
    \FOR{$i=1$ {\bfseries to} \texttt{max\_iter}}
        \STATE $c = \frac{a + b}{2}$
        \IF{$f(c)=0$ or $c - a \leq \epsilon$}
            \STATE {\bfseries Return:} $c$
        \ENDIF
        \IF{$f(c) f(a) < 0$}
            \STATE $b = c$
        \ELSE
            \STATE $a = c$
        \ENDIF
    \ENDFOR
    \STATE {\bfseries Return:} the algorithm did not converge, increase $\texttt{max\_iter}$.
 \end{algorithmic}
 \end{algorithm}
 
Additionally, we explain below how this model fit interpolation can be used to obtain an alternative CP set. Note that its midpoint can also be used as a candidate for initialization.
For computational efficiency, one can rely predominantly on the fact that for the usual prediction problems in machine learning, it is unneeded to optimize below the inevitable statistical error \citep{Bousquet_Bottou08}. This means that a high optimization accuracy in the model fit might be unnecessary to achieve better generalization performances. Therefore, with a coarse optimization tolerance, we can preview the final shape of the conformity function. Whence, one can replace $\pi(z)$ by $\pi^{(\epsilon_0)}(z)$ which is computed with a rough optimization error $\epsilon_0$ in order to guess the shape of the function $\pi(\cdot)$. Similarly, if $\pi(z_0)$ fails to be a valid initialization, then decrease $\epsilon_0$, and repeat the process.
In all our experiments, it works fine after a very few number of iterations. Nonetheless, we do not have strategies to avoid worst-case situations or any mathematical guarantee in the total number of iterations needed to find a valid initialization.
We illustrate this strategy in \Cref{fig:initialization} in both situations of failure, and success.

\paragraph{Further complexity reduction.}
 
In cases where the regression map $z \mapsto \mu_z(x)$ for any feature $x$, can be traced with homotopy as in Ridge \citep{Nouretdinov_Melluish_Vovk01}, and Lasso \citep{Lei19}, it takes $O(n^2)$ to compute the exact conformal set. This can be reduced to $O(n\log n)$ by sorting the roots of the instance-wise scores $E_i(z) - E_{n+1}(z)$ for $i$ in $[n]$, and cleverly flattening the double loop when evaluating the ranks of the score functions \citep[Chapter 2.3]{Vovk_Gammerman_Shafer05}. By relaxing the exactness, none of these two steps is needed in our approach. We obtain an asymptotic improvement to $O(n\log_2(1/\epsilon))$, and an easier to implement algorithm. \\
 
When the model fit is parameterized by the solution of optimization problem in \Cref{eq:model_optimization}, the regularity of the loss function, and penalty terms play a major role in the computational tractability of the full conformal prediction set. Leveraging smoothness, and convexity assumptions on the loss or penalty functions, it has been shown in \citep{Ndiaye_Takeuchi19} that approximate solutions can be used without refitting the model for close candidates $z$. The resulting conformal set is $\overline{\Gamma}^{(\alpha, \epsilon)} = \{z \in \bbR:\,  \overline{\pi}(z, \epsilon) > \alpha \}$ where the corresponding typicalness function incorporates the optimization error\ie
$$ \overline{\pi}(z, \epsilon) = \frac{1}{n+1}\sum_{i=1}^{n+1} \mathbb{1}_{E_{i}(z) \geq E_{n+1}(z) - 2\sqrt{2 \nu \epsilon}} \enspace.$$
One can further show that the typicalness function based on exact solution $\hat\pi(\cdot)$ is uniformly upper bounded by $\overline{\pi}(\cdot, \epsilon)$, and then $\hat\Gamma^{(\alpha)} \subset \overline{\Gamma}^{(\alpha, \epsilon)}$. This can be equally used to reduce the number of model fit, and also to wrap the CP set based on exact solution by applying \texttt{rootCP} directly to $\overline{\pi}(\cdot, \epsilon)$ instead of computing a whole approximation path. \looseness=-1\\

Compared with the homotopy approach, \texttt{rootCP} will always make a smaller number of model fits. By merely storing each model evaluation, it benefits from warm-start boosting by employing the solutions of the previous function call. Alas, the homotopy approaches require either an exact solution or an approximate one with a strict control of the optimization error. This control is not always available if one does not provide a computable upper bound of it, for example by precisely evaluating the duality gap. Such bounds are hardly available in non-convex settings which greatly reduce its applicability in modern machine learning techniques. Meanwhile, the complexity of the approximate homotopy algorithm is $O(\frac{z_{\max} - z_{\min}}{\sqrt{\epsilon_0}})$ where $\epsilon_0$ is the optimization error of the model fit. Additionally to the linear dependence in the initial interval length, it cannot be launched for small $\epsilon_0$ whereas the number of model fit in the root-finding approach is not degraded. In a nutshell, the proposed method avoids the computation of the whole path. Hence, it enjoys an exponential improvement over the homotopy approach \wrt to the initial interval length $z_{\max} - z_{\min}$, and an overall complexity that is independent of $\epsilon_0$. It can then be used with highly optimized model fit where the homotopy method cannot even be launched.
 
\paragraph{Drawbacks.}
 
Full conformal prediction set is \emph{not always} an interval. When it is a union of few well separated intervals, our proposed method cannot be applied without finely bracketing these intervals. One can include a human in the loop. The discrete function $\pi^{(\epsilon_0)}$ offers a cheap pre-visualization of the landscape of the conformity function that allows to detect these situations, and infer a proper bracketing. At this point, efficiently enumerating all the roots remains a challenging task that we leave as an open problem. In the following proposition, we provide a sufficient condition so that the conformal set is an interval. \eug{It essentially consists of a simple condition so that the conformity function is monotonically increasing until it reaches its maximum value, and then monotonically decreasing.}
 
\begin{proposition}\label{prop:quasiconcavity_test}
If for any $i$ in $[n]$, the difference of instance-wise error function $z \mapsto E_i(z) - E_{n+1}(z)$ is quasi-concave, and has two zeros $a_i \leq b_i$ such that
$$ \max_{i \in [n]} a_i \leq \min_{i \in [n]} b_i \enspace, $$
then the typicalness function $\pi(\cdot)$ is quasi-concave, and the conformal prediction set at a level $\alpha \in (0, 1)$ is either empty or an interval.
\end{proposition}
 
\begin{proof}
\eug{The function $\psi_i(z) := E_i(z) - E_{n+1}(z)$ is quasi-concave implies that its $0$-level set is convex, and $\{z:\psi(z) \leq 0\} = (-\infty, a_i] \cup [b_i, +\infty)$ or empty. Whence, $$F_{n+1}(E_{n+1}(z)) = \frac{1}{n+1}\sum_{i=1}^{n+1}\mathbb{1}_{(-\infty, a_i] \cup [b_i, +\infty)}(z) \enspace,$$ where, without loss of generality, we assume that all intervals are non-empty since they have a zero contribution to the sum. Now, the condition $\max_{i \in [n]} a_i \leq \min_{i \in [n]} b_i$ implies that the function $z \mapsto \pi(z) = 1 - F_{n+1}(E_{n+1}(z))$ is monotonically increasing in $(-\infty, a_{(n+1)}]$, and decreasing in $[b_{(1)}, +\infty)$. Hence the result.}
\begin{flushright}
    $\blacksquare$
\end{flushright}\end{proof}
 
\Cref{prop:quasiconcavity_test} generalizes \citep[Theorem 3.3]{Lei19} which provides a sufficient condition so that the Lasso conformal set is an interval. Unfortunately, such sufficient conditions are not testable for most problems. Indeed, it requires knowing all the zero crossing points of the function $z \mapsto E_i(z) - E_{n+1}(z)$ for all indices $i \in [n]$ which is as hard as computing the whole function $z \mapsto \pi(z)$.\\
 
\eug{In the literature, similar assumptions are made to obtain a conformal predictive distribution by enforcing a monotonicity on the score functions, see \citep{Vovk_Shen_Manokhin_Xie17}. We remind that even when the typicalness function $z \mapsto \pi(z)$ is not quasi-concave, our algorithm is still valid as long as the conformal set is an interval; which is a much weaker assumption than quasi-concavity. It would be interesting to study in more detail how to characterize the class of score function that systematically leads to a CP set being an interval. This is not necessarily obvious since it is easy to build an adversarial example. Indeed, consider a sinusoidal score function
$S(a, b) = |\sin(a - b)|$. It treats the data symmetrically across the instances, and then satisfies all the assumptions. The corresponding CP set is a union of an infinite number of intervals even for very simple regression models such as least squares. Which is a reminiscence of \texttt{No-Free-Lunch}: without any assumption, computations of a CP set is impossible, even with splitting approaches!}
 
\begin{figure*}
\center
\includegraphics[width=0.49\linewidth]{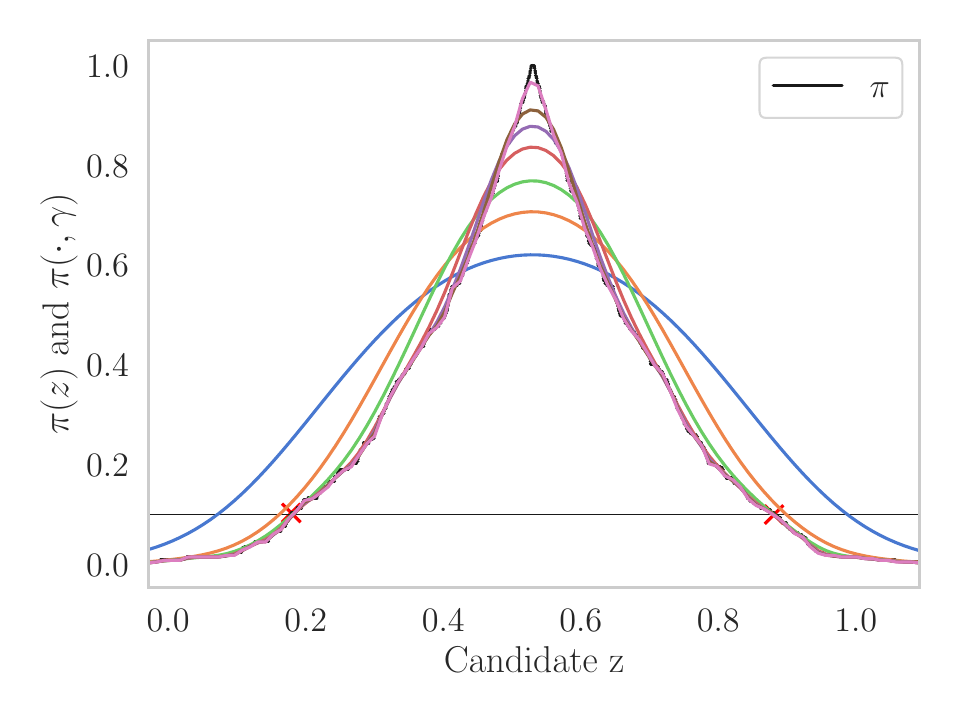}
\includegraphics[width=0.49\linewidth]{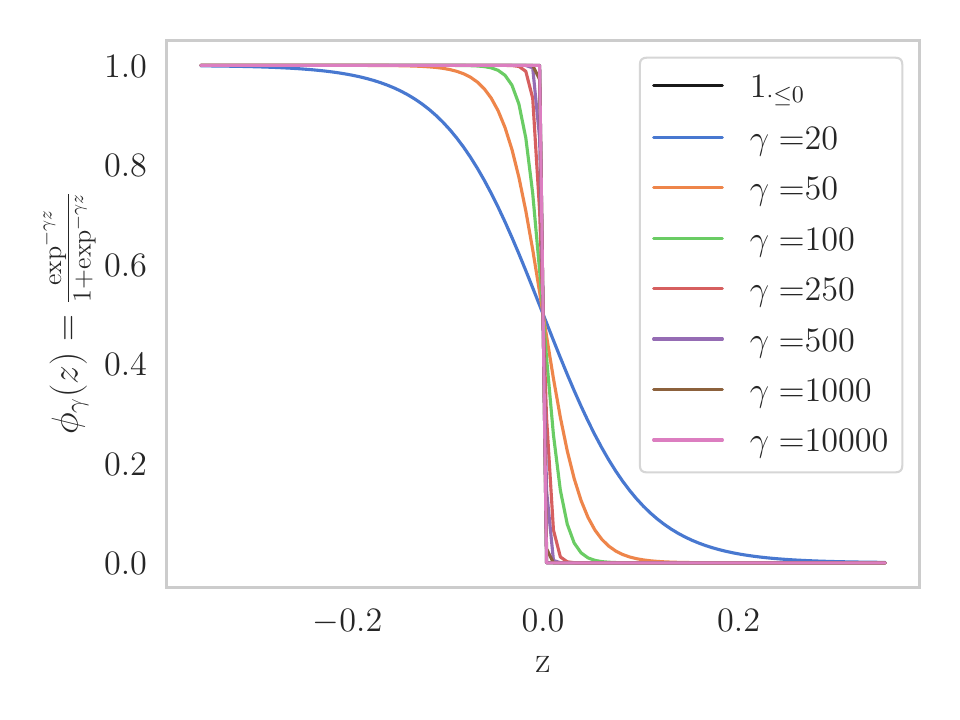}
\includegraphics[width=0.49\linewidth]{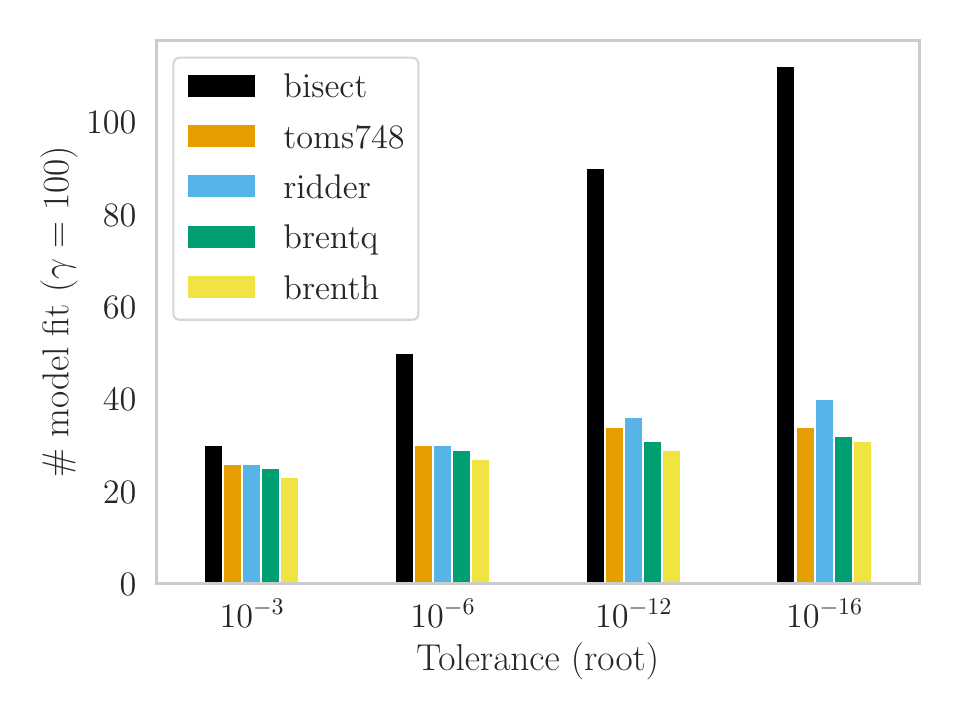}
\includegraphics[width=0.49\linewidth]{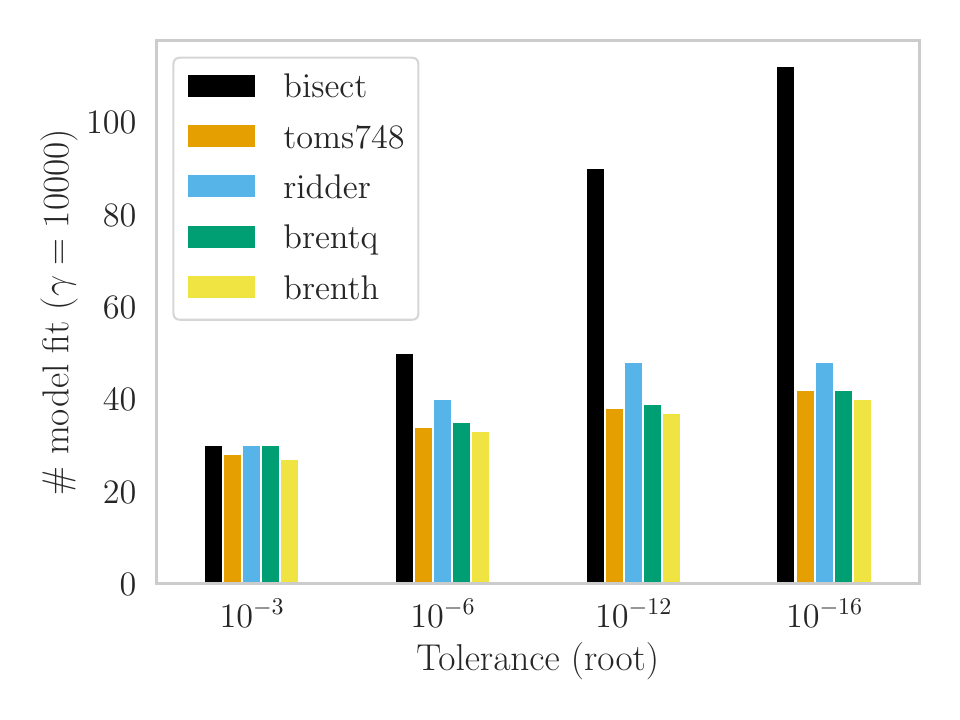}
\includegraphics[width=0.49\linewidth]{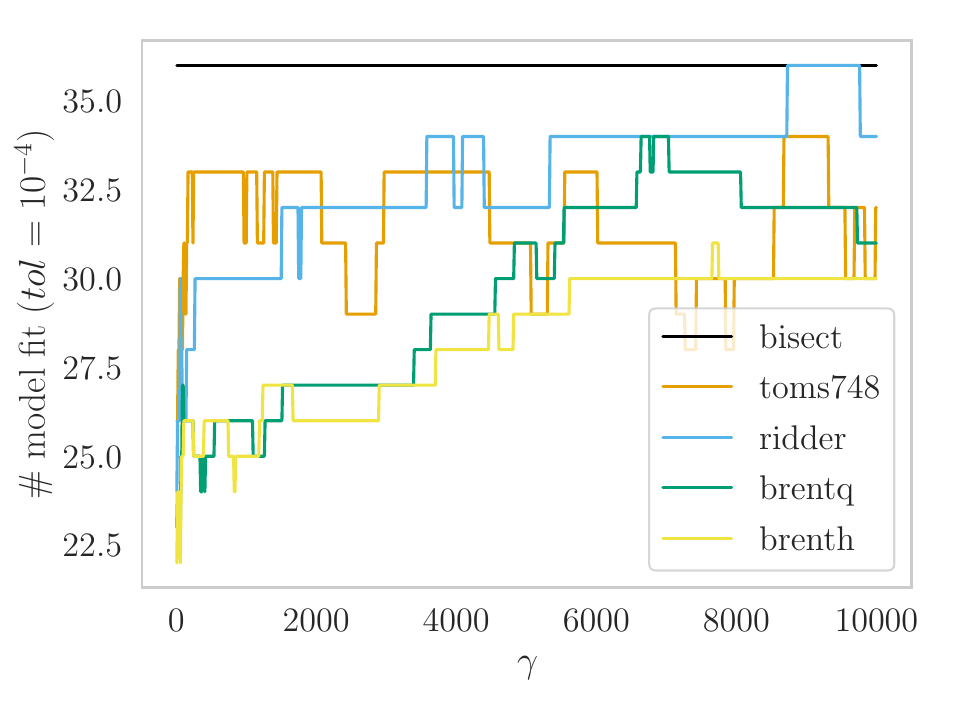}
\includegraphics[width=0.49\linewidth]{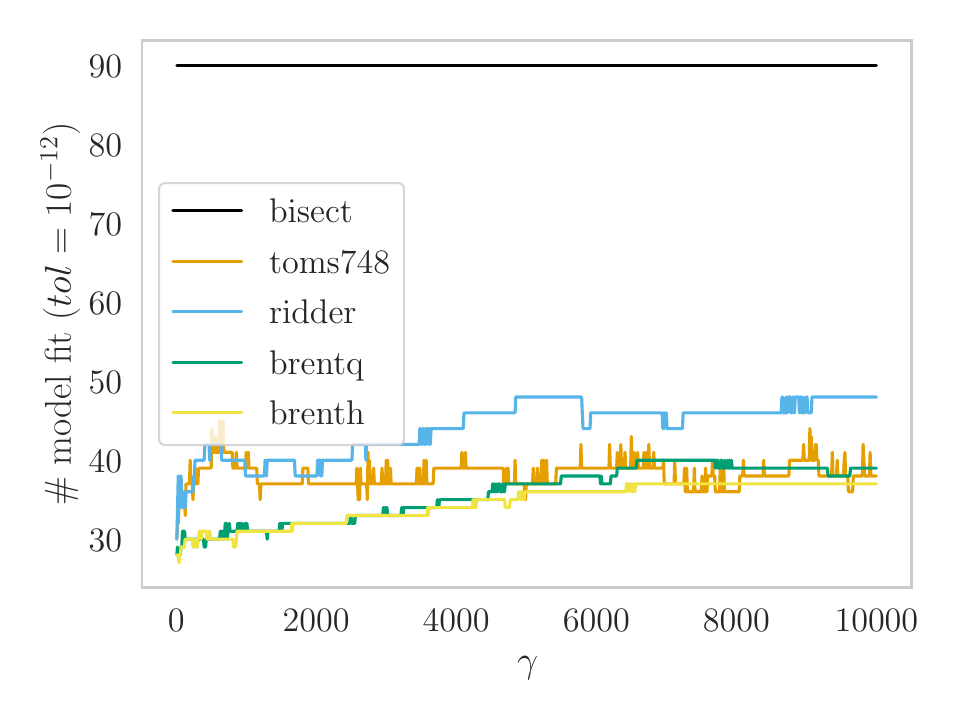}
\caption{Illustration of the smoothed conformal set with data generated from $\texttt{sklearn}$ as $X, y =$ \texttt{make\_regression}$(n=300, p=50)$. The smoothed typicalness function $\pi(\cdot, \gamma)$ is evaluated with several values for the hyperparameter $\gamma$. The underlying estimator is the ridge regressor with parameter $\lambda = p/\normin{\beta_{\mathrm{LS}}}^2$ where $\beta_{\mathrm{LS}}$ is the Least-squares estimator on the observed dataset $\mathcal{D}_n$.\looseness=-1 \label{fig:smoothed_CP}}
\end{figure*}
 
\subsection{Interpolated Conformal Prediction}
\label{subsec:Interpolated_Conformal_Prediction}
 
The full conformal prediction set is computationally expensive since it requires knowing exactly the map $z \mapsto \mu_z(\cdot)$. The splitting approach does not use all the data in the learning phase but is computationally efficient since it requires a single model fit. Alternatively, it was proposed in \citep{Lei_GSell_Rinaldo_Tibshirani_Wasserman18} to use an arbitrary discretization, and its theoretical analysis in \citep{Chen_Chun_Barber18} unfortunately failed to preserve the coverage guarantee. In this section, we argue that grid based strategy with an interpolation point of view, stands as an "in-between" strategy that exploits full data with a restricted computational time while preserving the coverage guarantee. We propose to compute a conformal prediction set based on an interpolation of the model fit map given a finite number of query points. The main insight is that the underlying model fit plays a minor role in the coverage guarantee; the only requirement is to be symmetric with respect to permutation of the data. As such, the model path $z \mapsto \hat\mu_z(\cdot)$ can be replaced by an interpolated map $z \mapsto \tilde\mu_z(\cdot)$ based on query points $z_1, \cdots, z_d$. It reads to a valid prediction set as long as the interpolation preserves the symmetry \footnote{Otherwise one can always perform a symmetrization\eg using a model parameter $\tilde\beta(z) = \frac{1}{(n+1)!} \sum_{\sigma \in \Sigma_{n+1}} \beta(w_{\sigma(1)}, \cdots, w_{\sigma(n)}, w_{\sigma(n+1)})$,
where $w_i = (x_i, y_i)$ if $i$ in $[n]$, $w_{n+1} = (x_{n+1}, z)$, and $\Sigma_{n+1}$ is the group of permutation of $[n+1]$.}.
For instance, one can rely on a piecewise linear interpolation

\begin{equation}\label{eq:piece_linear_interpolation}
\tilde\mu_{z} =
\begin{cases}
\frac{z_1 - z}{z_1 - z_{\min}} \hat\mu_{z_{\min}} + \frac{z_{\min} - z}{z_1 - z_{\min}} \hat\mu_{z_1} &\text{ if } z \leq z_{\min} \enspace, \\
\frac{z - z_{t+1}}{z_t - z_{t+1}} \hat\mu_{z_t} + \frac{z - z_{t}}{z_{t+1} - z_t} \hat\mu_{z_{t+1}} & \text{ if } z \in [z_t, z_{t+1}]\enspace, \\
\frac{z - z_d}{z_{\max} - z_d} \hat\mu_{z_{\max}} + \frac{z_{\max} - z}{z_{\max} - z_d} \hat\mu_{z_d} &\text{ if } z \geq z_{\max}\enspace,
\end{cases}
\end{equation}
where $\hat\mu_z(x)$ is a prediction map trained on the augmented dataset $\mathcal{D}_{n+1}(z)$ as in \Cref{eq:model_optimization}.

As before, one defines the instance-wise score functions
\begin{align*}
\forall i \in [n],\,\tilde E_{i}(z) = S(y_i,\, \tilde\mu_z(x_{i})) \text{ and }
\tilde E_{n+1}(z) = S(z,\, \tilde\mu_z(x_{n+1})) \enspace.
\end{align*}
The conformal set based on interpolated model fit is then defined as
\begin{align*}
\tilde \Gamma^{(\alpha)}(x_{n+1}) &= \{z: \tilde\pi(z) \geq \alpha \}, \text{ where }\\
\tilde\pi(z) &= 1 - \frac{1}{n+1}\sum_{i=1}^{n+1} \1_{\tilde E_i(z) \leq \tilde E_{n+1}(z)}\enspace.
\end{align*}
Since the map $z \mapsto \tilde\mu_z(\cdot)$ is (or can be made) symmetric, it is immediate to see that $\tilde \Gamma^{(\alpha)}(x_{n+1})$ is a valid conformal set following the same proof technique.\\

We remind that the conformal set can be highly concentrated around its midpoint, and the typicalness of most candidates is close to zero. Whence, we suggest restricting the query points around an estimate of the conformal set provided by a localization step. \eug{Also, it could be interesting to evaluate the performance of more sophisticated interpolation methods like splines  in order to have more symbiosis between the interpolation, and the smoothing of the rank function introduced in the next section. Nevertheless, the simplicity of linear interpolation allows an exact calculation of the conformal prediction set because one can easily enumerate the change points of the rank function. This is not necessarily preserved with higher order interpolation, and requires further investigation.}
 
\begin{remark}[Interpolation of the typicalness map]
Given the query points, and their corresponding typicalness $(z_1, \pi(z_1)), \cdots, (z_d, \pi(z_d))$, one can also directly learn a function that approximate the typicalness $z \mapsto \pi(z)$. However, in this case, we could not establish the theoretical coverage guarantee of this method. Moreover, when the conformal set is highly localized, most of the $\pi(z_i)$ might be close to zero leading to a flat, and poorly interpolated typicalness map.
\end{remark}
 
Previous discretization approaches did not preserve the coverage guarantee or did it at expensive cost by approximating a model fit path on a wide range, with a high precision at every step, and was restricted to convex problems. The interpolation point of view that we provided allows us to compute a valid conformal prediction set with arbitrary discretization without loss in the coverage guarantee, and without restriction to convex problems. Also note that depending on the interpolation used, there is no need to assume that the conformal prediction set is an interval. Indeed, in the case below of piecewise linear interpolation, one can easily enumerate all the change points of the conformal function as in homotopy methods. However, in general we also recommend the use of \texttt{rootCP}. Finally, note that in the case of the ridge estimator (which is linear in $z$), the exact conformal prediction coincides with that of the interpolation.

\subsection{Smoothed Conformal Prediction}
 
Conformal prediction sets rely on rank computations. The latter function is piecewise constant, and has no useful first order information in the sense that it is either null or undefined. We propose a smooth approximation of the typicalness function to reduce the number of query points. In addition to exchangeability, we merely use the fact that $F_{n+1}$ is increasing, and the linearity of the sum to obtain the coverage guarantee. Likewise, one should be able to replace $\mathbb{1}_{E_i - z \leq 0}$ with a continuously differentiable, and increasing function $\phi_{\gamma}(E_i - z)$. Hence, replacing the function $z \mapsto F_{n+1}(E_{n+1}(z))$ by a smoother one allows the use of more efficient gradient or quasi-Newton-based root finding methods. We further investigate the influence of such smoothing on the coverage guarantee. In practice, we simply choose the sigmoid function $\phi_{\gamma}(x) = \frac{ \mathrm{e}^{-\gamma x} }{1 + \mathrm{e}^{-\gamma x}}$ as in \citep{Qin_Liu_Li10}. We have
\begin{align*}
\mathrm{Rank}(u_j) := \sum_{i=1}^{n+1} \mathbb{1}_{u_i - u_j\leq 0}
\approx \sum_{i=1}^{n+1} \phi_{\gamma}(u_i - u_j) =: \mathrm{sRank}(u_j, \gamma) \enspace.
\end{align*}
The main advantage is that the map $\mathrm{sRank}(\cdot, \gamma)$ improves the regularity of $\mathrm{Rank}(\cdot)$, and allows faster convergence.

The smooth approximation of the typicalness function is then defined as
\begin{align*}
\pi(z) &\approx \pi(z, \gamma) := 1 - \frac{1}{n+1} \mathrm{sRank}(E_{n+1}(z), \gamma) \enspace,
\end{align*}
and the smoothed conformal prediction set (illustrated in \Cref{fig:smoothed_CP}) as
\begin{equation*}
\Gamma^{(\alpha, \gamma)}(x_{n+1}) = \{z: \pi(z, \gamma) > \alpha \} \enspace.
\end{equation*}
Now computing an approximation of the conformal prediction set is equivalent to finding the smallest, and largest solution of the equation $\pi(z, \gamma) = \alpha$ which is often easier to solve than $\pi(z) = \alpha$. Using different root-finding solvers, we illustrate the computational advantages by displaying the reduction of the number of model fit in \Cref{fig:smoothed_CP}.
 
\begin{remark}[Gradient Based Solvers]
When, for any feature $x$, the regression map $z \mapsto \mu_z(x)$ is differentiable, the solutions of equation $\pi(z, \gamma) = \alpha$ could be approximated with more efficient gradient based root-finding algorithm. However, the function $z \mapsto \pi(z, \gamma)$ is mostly flat except at a tiny vicinity of the conformal set which makes the convergence difficult unless a good initialization is found. One could also rely on a regularized version by minimizing $(\pi(z, \gamma) - \alpha)^2 + \tau z^2$ which requires a proper tuning of the hyper parameter $\tau$. Both of these strategies turn out to be less stable, and need further investigations.
\end{remark}
 
\begin{remark} \eug{Here it is clear that the term \textit{"smooth"} refers to the differentiability of the conformity function; and was introduced for computational reasons. This is not to be confused with the \textit{"Smoothed Conformal Predictors"} introduced in \citep[Page 27]{Vovk_Gammerman_Shafer05} where the borderline cases $E_i(z) = E_{n+1}(z)$ are treated more carefully\ie \textit{smoothly} penalized with a random parameter between $0$, and $1$ instead of increasing the rank with $1$. This randomization essentially breaks the ties to ensure an exact coverage guarantee. All the computational methods we introduce in this article apply immediately to this case.}
\end{remark}

We analyze the statistical consequences of using a continuous version of the indicator function. We recall the definition of the smoothed version of the empirical cumulative distribution, and empirical quantile:
\begin{align*}
&\tilde F_{n+1}(z) = \frac{1}{n+1}\sum_{i=1}^{n+1} \phi_\gamma(E_i - z),
&\tilde Q_{n+1}(\alpha) = \inf\{z \in \bbR: \tilde F_{n+1}(z) \geq \alpha\} \enspace.
\end{align*}

\begin{proposition}[Coverage guarantee of the smooth relaxation]\label{prop:smooth_coverage}
For a sequence of exchangeable random variable $E_1, \ldots, E_{n+1}$, it holds for any $\tilde\alpha$ in $(0, 1)$,
$$\mathbb{P}(\tilde F_{n+1}(E_{n+1}) \leq \tilde\alpha) \geq \tilde\alpha - \Delta(\gamma) \enspace,$$ where
 $\Delta(\gamma) = \sup_{x}(\phi_{\gamma} - \mathbb{1}_{\cdot \leq 0})(x)$.
\end{proposition}
 
\begin{proof}
By definition of $\tilde F_{n+1}$, and $\tilde Q_{n+1}$, we have
\begin{align*}
\tilde\alpha \leq \tilde F_{n+1}(\tilde Q_{n+1}(\tilde\alpha))
&= \frac{1}{n+1}\sum_{i=1}^{n+1} \phi_{\gamma}(E_i - \tilde Q_{n+1}(\tilde\alpha))\\
&= \frac{1}{n+1}\sum_{i=1}^{n+1} \mathbb{1}_{E_i \leq \tilde Q_{n+1}(\tilde\alpha)} + (\phi_{\gamma} - \mathbb{1}_{\cdot \leq 0})(E_i - \tilde Q_{n+1}(\tilde\alpha))\\
&\leq \frac{1}{n+1}\sum_{i=1}^{n+1} \mathbb{1}_{\tilde F_{n+1}(E_i) \leq \tilde\alpha} + \Delta(\gamma) \enspace.
\end{align*}
We conclude by taking the expectation on both sides along with exchangeability.
\begin{flushright}
    $\blacksquare$
\end{flushright}\end{proof}
 
In order to display a probabilistic statement, one needs to maintain the indicator function when defining the typicalness function. Replacing it with a continuous version will distort the coverage guarantee as described in \Cref{prop:smooth_coverage}. To obtain a well $\alpha$-calibrated confidence set, one must take into account such approximation error by choosing $(\tilde\alpha, \gamma)$ such that
\begin{equation*}
\tilde\alpha - \Delta(\gamma) \geq \alpha \enspace.
\end{equation*}
If $\tilde \alpha$ is fixed, one needs to be careful when choosing $\gamma$. Otherwise, we obtain a vacuous upper bound, and all the coverage guarantee is lost. Meanwhile, if $\gamma$ is chosen such that $\phi_{\gamma}$ is a lower approximation of the indicator function, then $\tilde\alpha$ can be taken as $\alpha$, and there is no calibration loss. However, when $\Delta(\gamma)$ is close to zero, $\phi_{\gamma}$ will be flat almost everywhere, and we will not get useful first order information. This brings a \emph{trade-off} between number of model fitting (which influences the computational time), and efficiency\ie length of the interval (wider $\tilde\alpha$-level set).
 
\paragraph{Building a gap.}
 
To finely assess how the vanilla conformal, and smoothed conformal set can be related in practice, one can simply design both a lower, and upper approximation of the indicator function\ie $\phi_{\gamma}^{+}$ and $\phi_{\gamma}^{-}$.

In that case, it is easy to see that
\begin{equation*}
\ell_{\gamma}^{+} \leq \ell_{\alpha}(x_{n+1}) \leq \ell_{\gamma}^{-} \text{ and } u_{\gamma}^{-} \leq u_{\alpha}(x_{n+1}) \leq u_{\gamma}^{+} \enspace,
\end{equation*}
which is equivalent to
$$
\Gamma_{\gamma}^{(\alpha, -)}(x_{n+1}) \subset \Gamma^{(\alpha)}(x_{n+1}) \subset \Gamma_{\gamma}^{(\alpha, +)}(x_{n+1}) \enspace.
$$
The overall complexity is moderately expanded (we now need to compute two different conformal prediction sets), and not too time consuming as long as the underlying model fit is reasonably computable.

\section{Experiments}
\label{sec:experiments}
 
\begin{figure}
\centering
\subfigure{\includegraphics[width=0.49\textwidth]{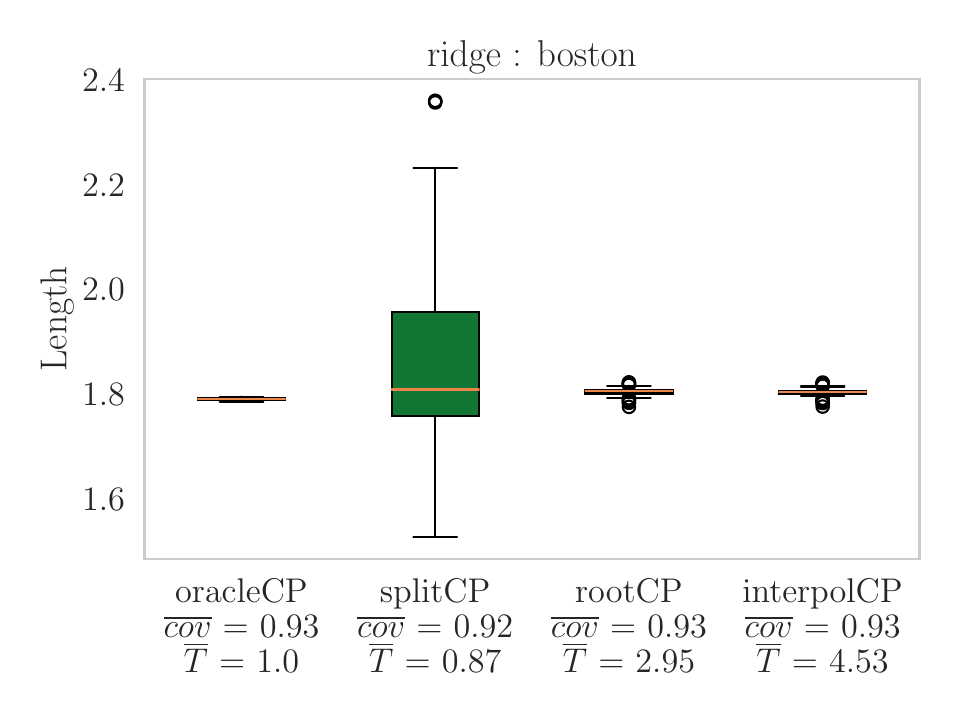}}
\subfigure{\includegraphics[width=0.49\textwidth]{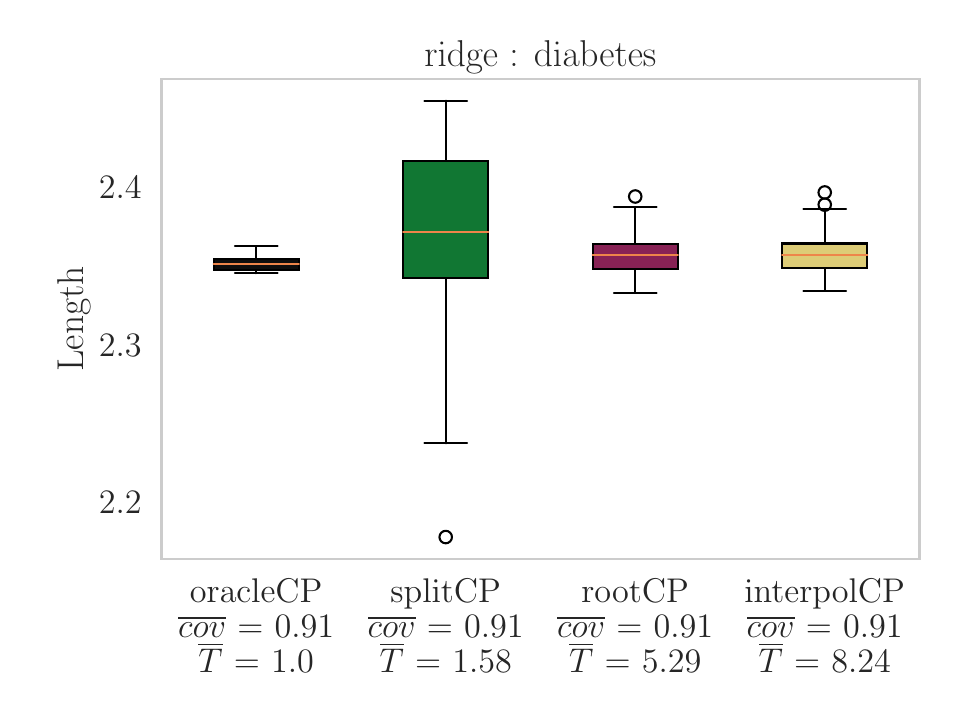}}
\subfigure{\includegraphics[width=0.49\textwidth]{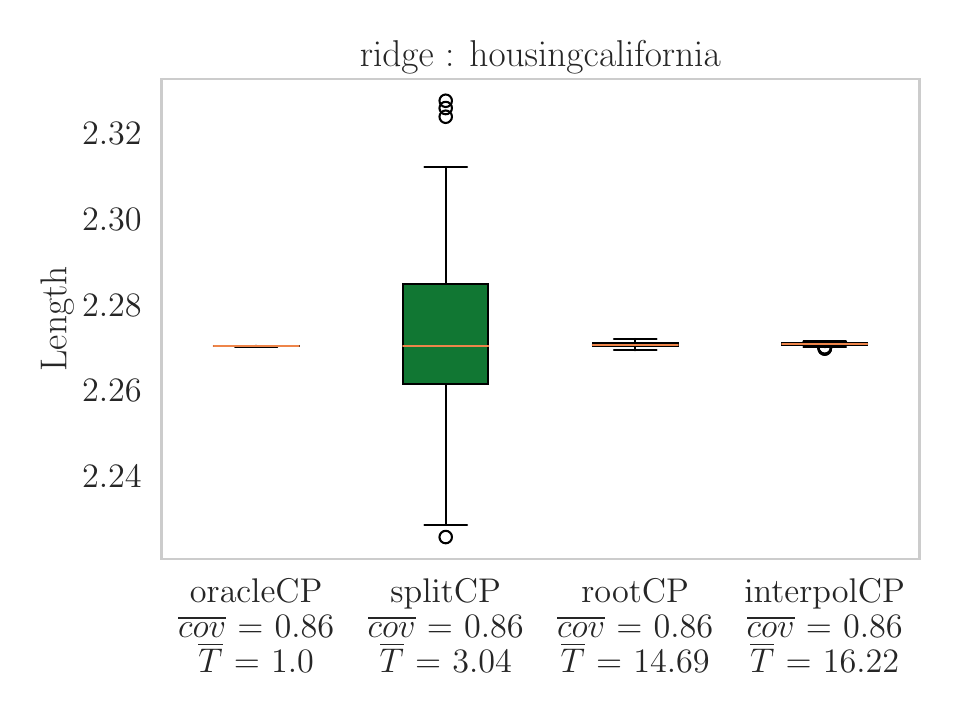}}
\subfigure{\includegraphics[width=0.49\textwidth]{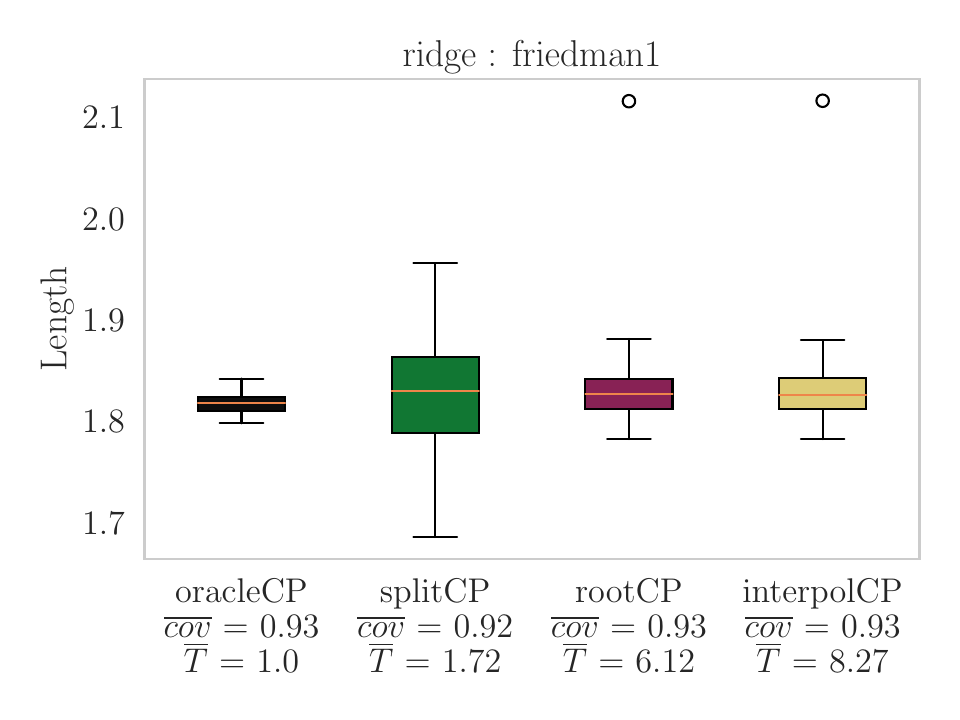}}
\caption{Benchmarking conformal sets for ridge regression models on real datasets. We display the lengths of the confidence sets over $100$ random permutation of the data. We denoted $\overline{cov}$ the average coverage, and $\overline{T}$ the average computational time normalized with the average time for computing \texttt{oracleCP} which requires a single model fit on the whole data. \label{fig:ridge_benchmarks}\looseness=-1}
\end{figure}
\begin{figure}
\subfigure{\includegraphics[width=0.49\textwidth]{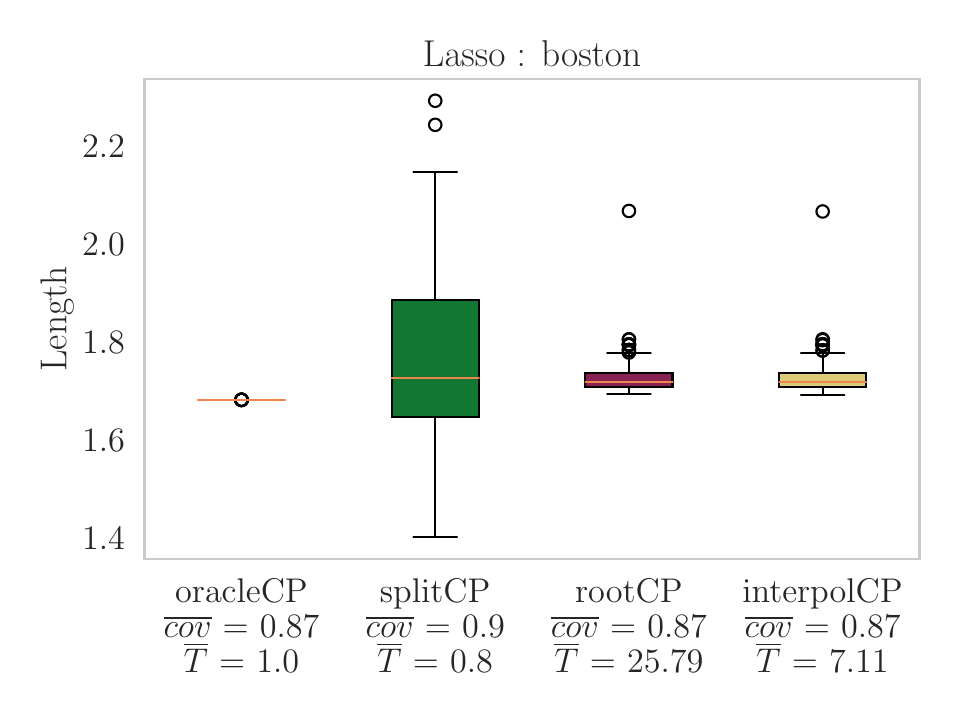}}
\subfigure{\includegraphics[width=0.49\textwidth]{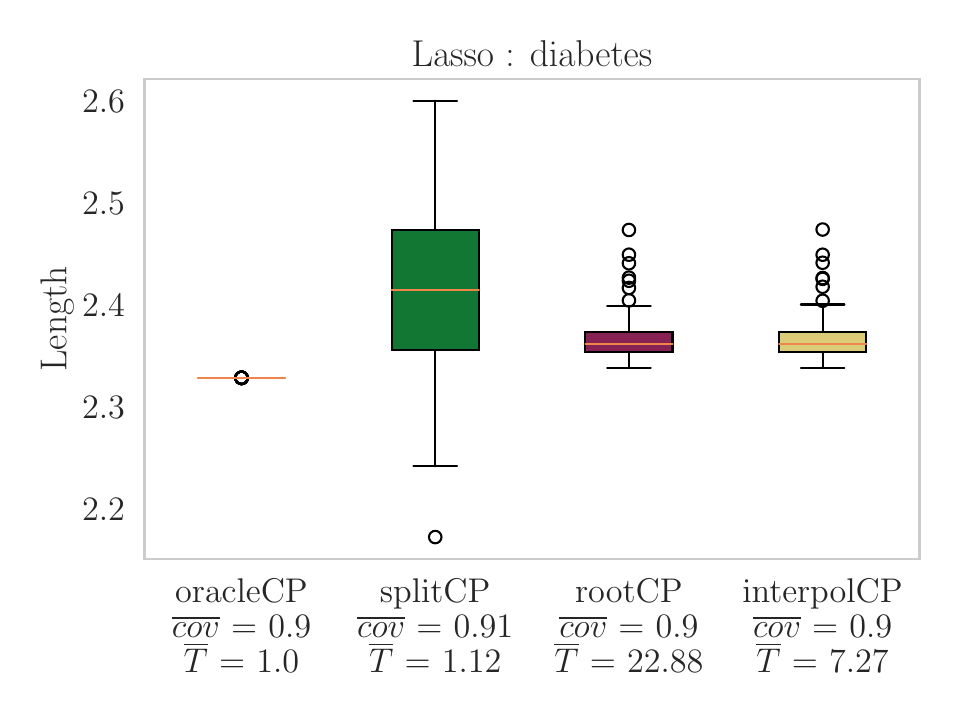}}
\subfigure{\includegraphics[width=0.49\textwidth]{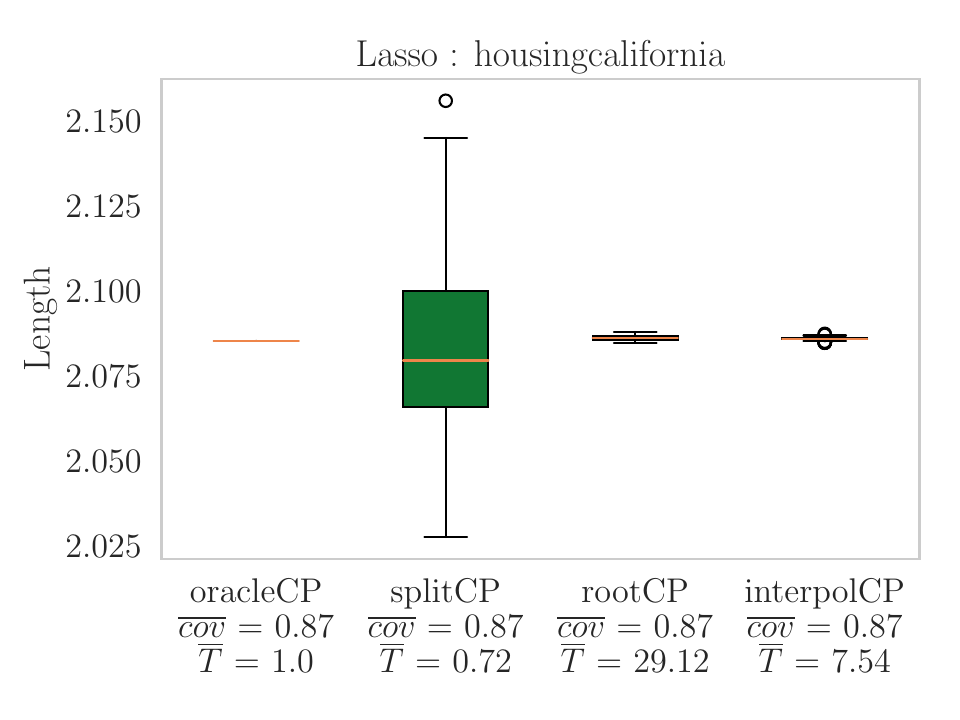}}
\subfigure{\includegraphics[width=0.49\textwidth]{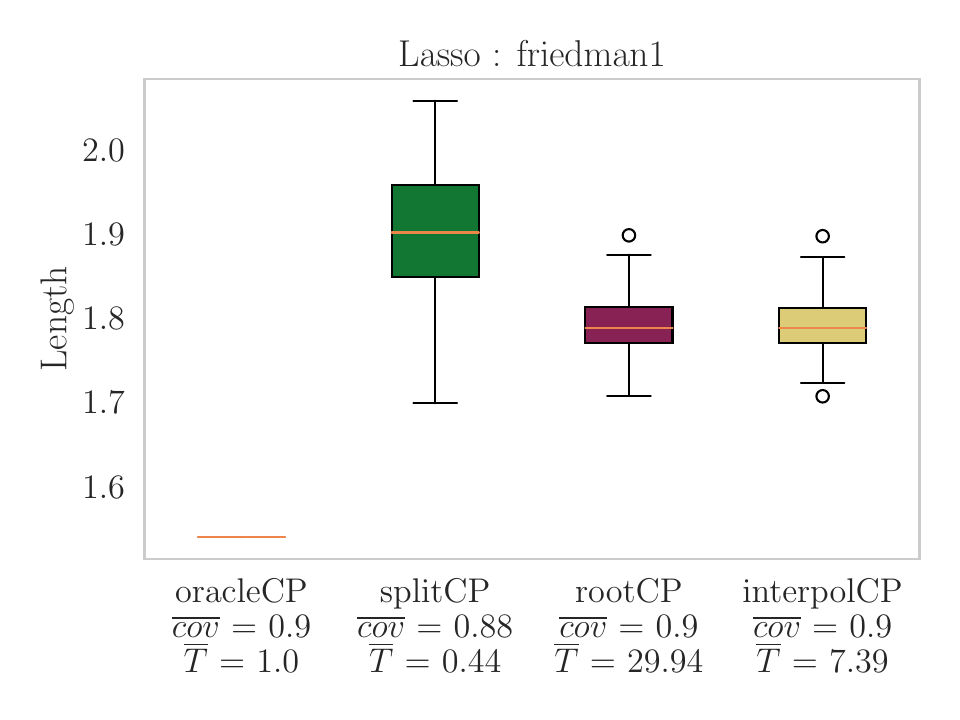}}
\caption{Benchmarking conformal sets for Lasso regression models on real datasets. We display the lengths of the confidence sets over $100$ random permutation of the data. We denoted $\overline{cov}$ the average coverage, and $\overline{T}$ the average computational time normalized with the average time for computing \texttt{oracleCP} which requires a single model fit on the whole data.\label{fig:Lasso_benchmarks} \looseness=-1}
\end{figure}
%
\begin{figure}
\subfigure{\includegraphics[width=0.49\textwidth]{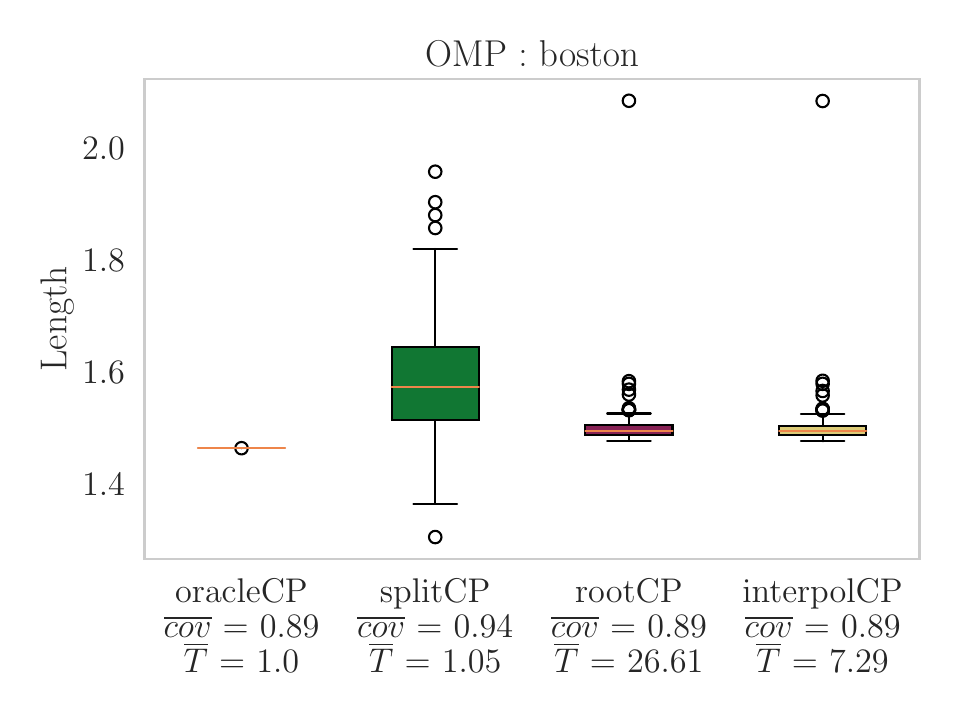}}
\subfigure{\includegraphics[width=0.49\textwidth]{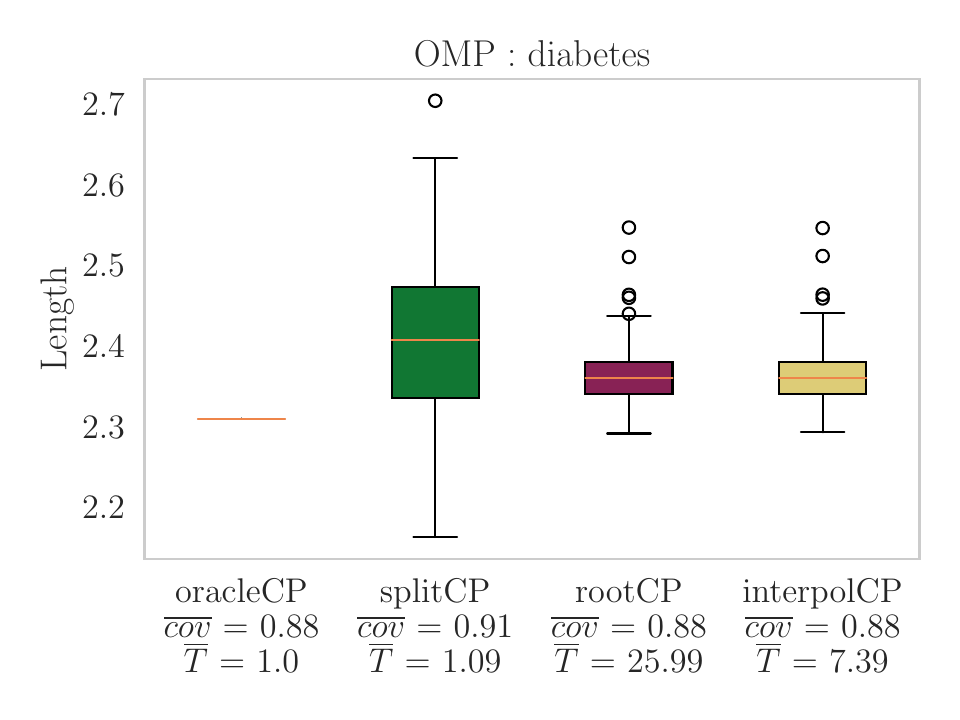}}
\subfigure{\includegraphics[width=0.49\textwidth]{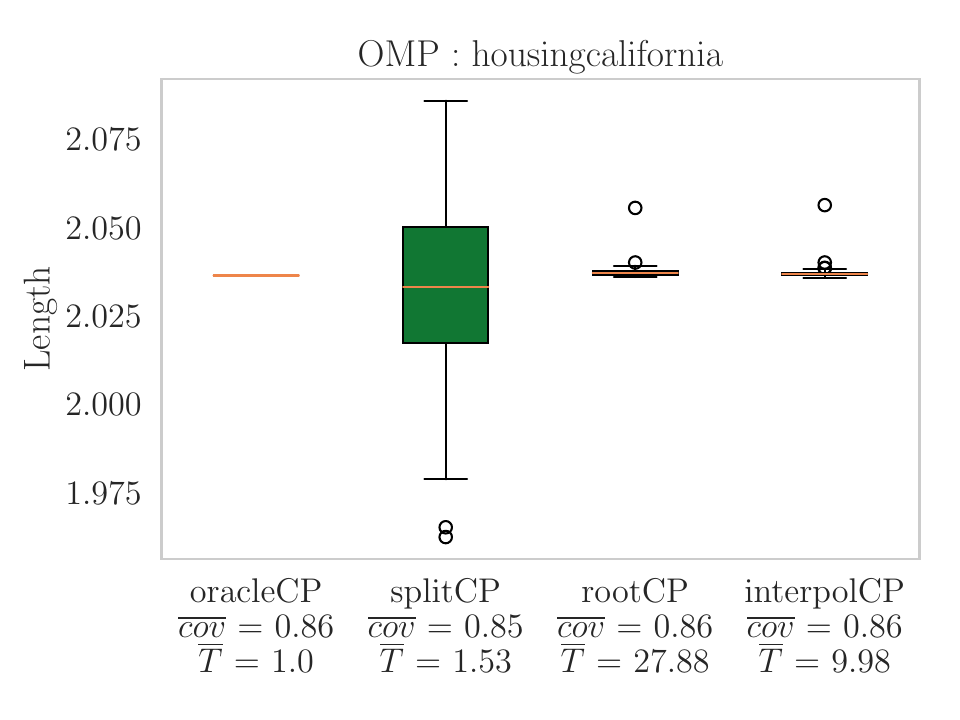}}
\subfigure{\includegraphics[width=0.49\textwidth]{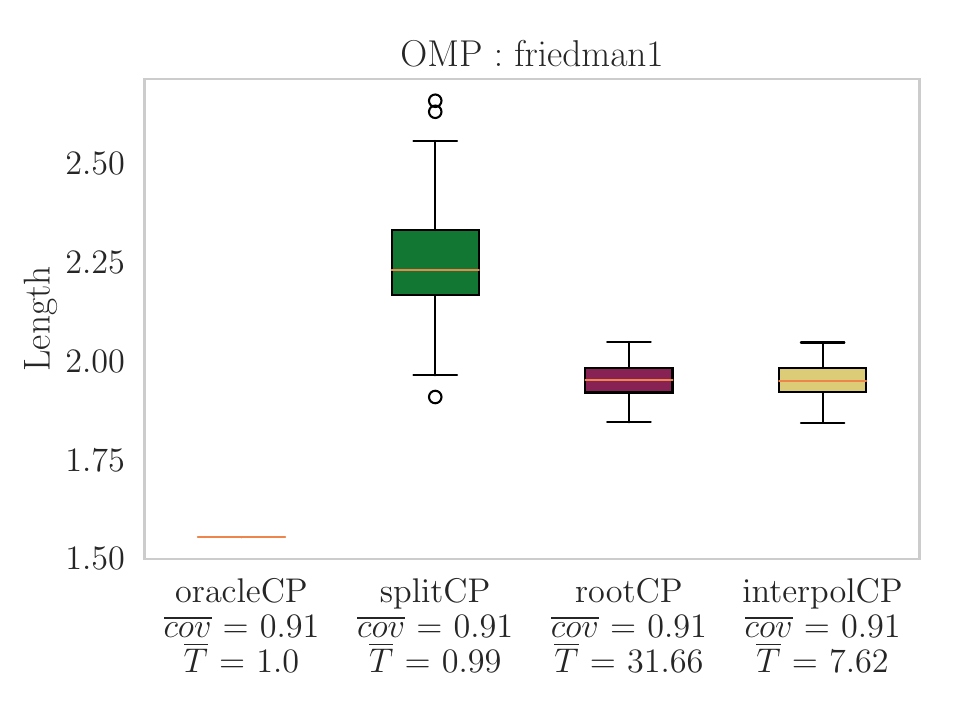}}
\caption{Benchmarking conformal sets for orthogonal Matching Pursuit regression models on real datasets. We display the lengths of the confidence sets over $100$ random permutation of the data. We denoted $\overline{cov}$ the average coverage, and $\overline{T}$ the average computational time normalized with the average time for computing \texttt{oracleCP} which requires a single model fit on the whole data.\label{fig:OMP_benchmarks} \looseness=-1}
\end{figure}
\begin{figure}
\subfigure{\includegraphics[width=0.49\textwidth]{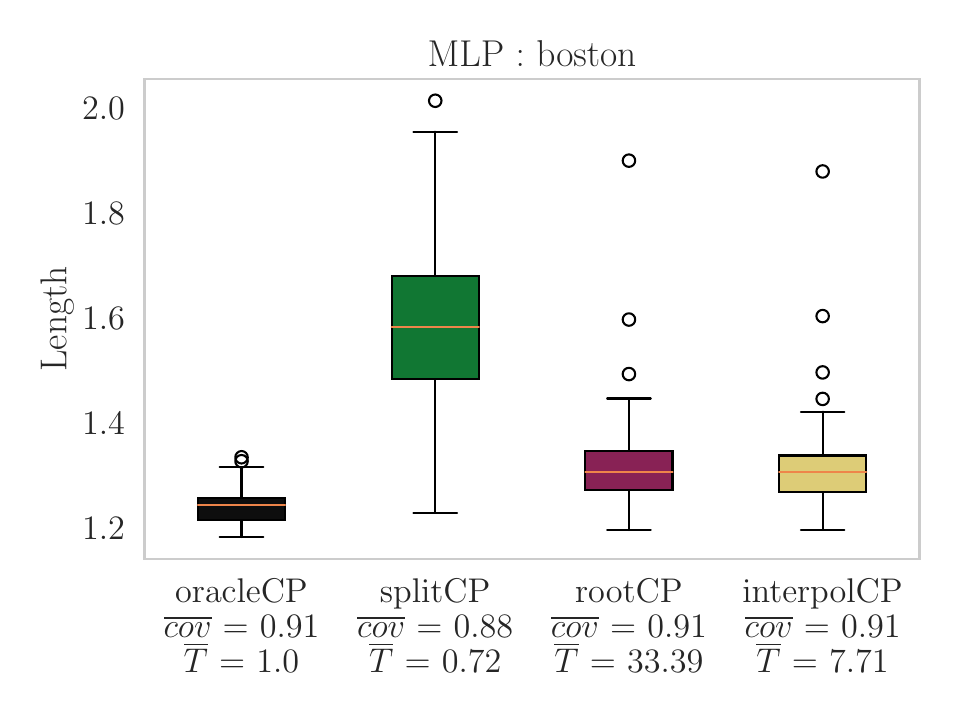}}
\subfigure{\includegraphics[width=0.49\textwidth]{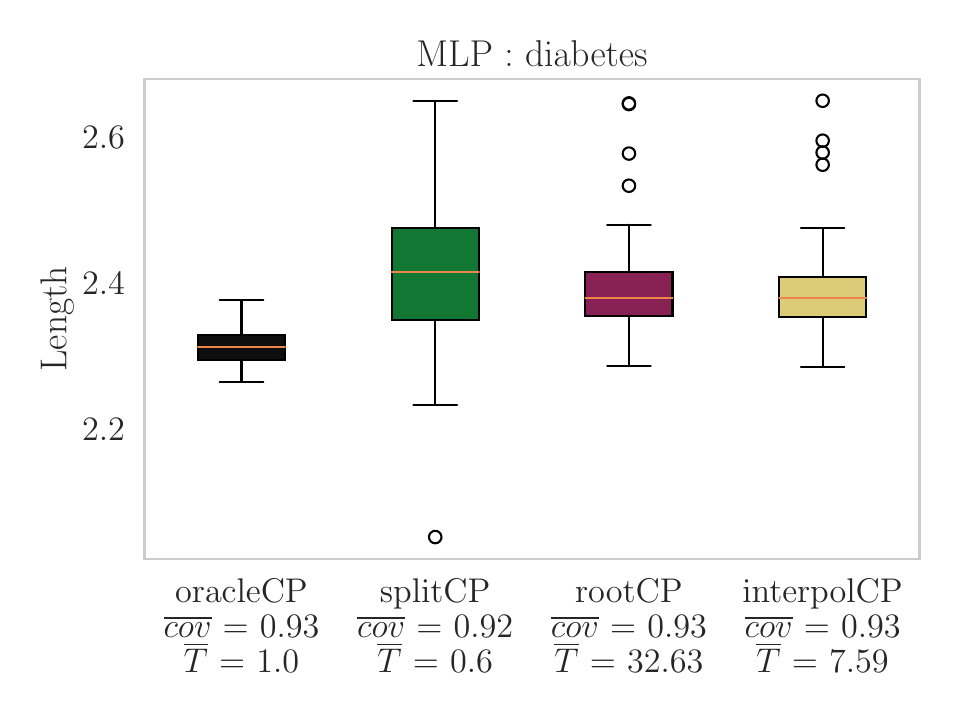}}
\subfigure{\includegraphics[width=0.49\textwidth]{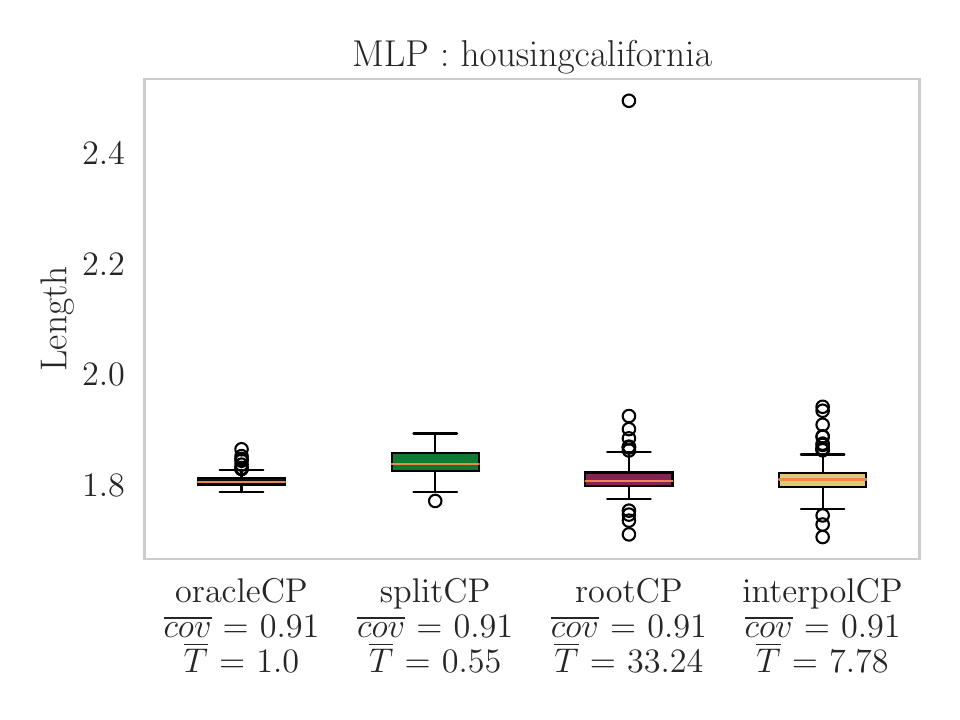}}
\subfigure{\includegraphics[width=0.49\textwidth]{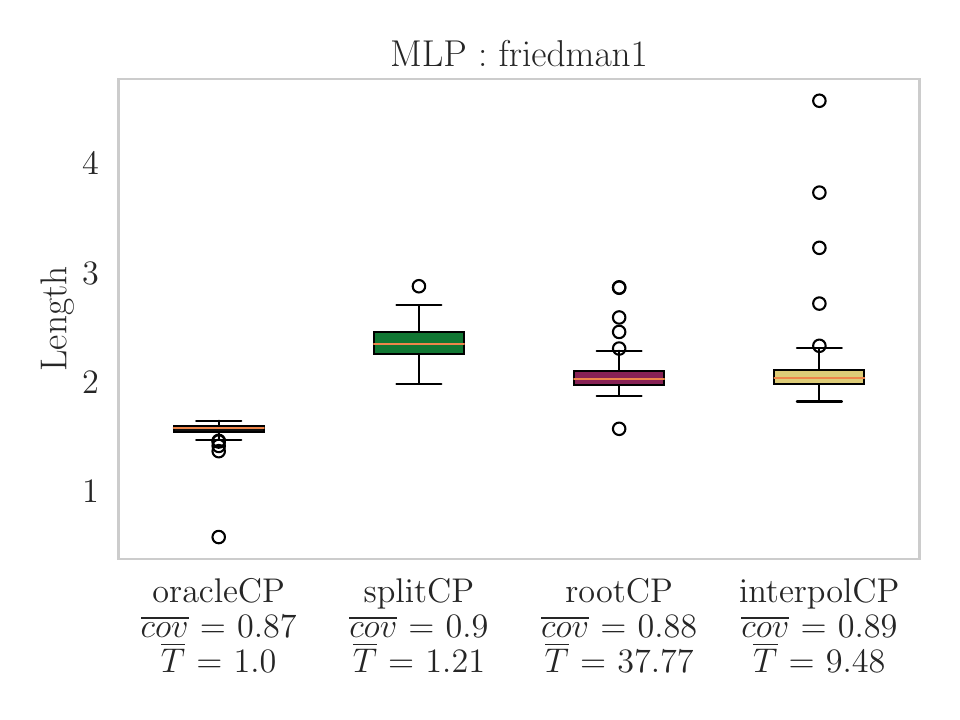}}
\caption{Benchmarking conformal sets for Multi-layer Perceptron regression models on real datasets. We display the lengths of the confidence sets over $100$ random permutation of the data. We denoted $\overline{cov}$ the average coverage, and $\overline{T}$ the average computational time normalized with the average time for computing \texttt{oracleCP} which requires a single model fit on the whole data. \label{fig:MLP_benchmarks}\looseness=-1}
\end{figure}
%
\begin{figure}
\subfigure{\includegraphics[width=0.49\textwidth]{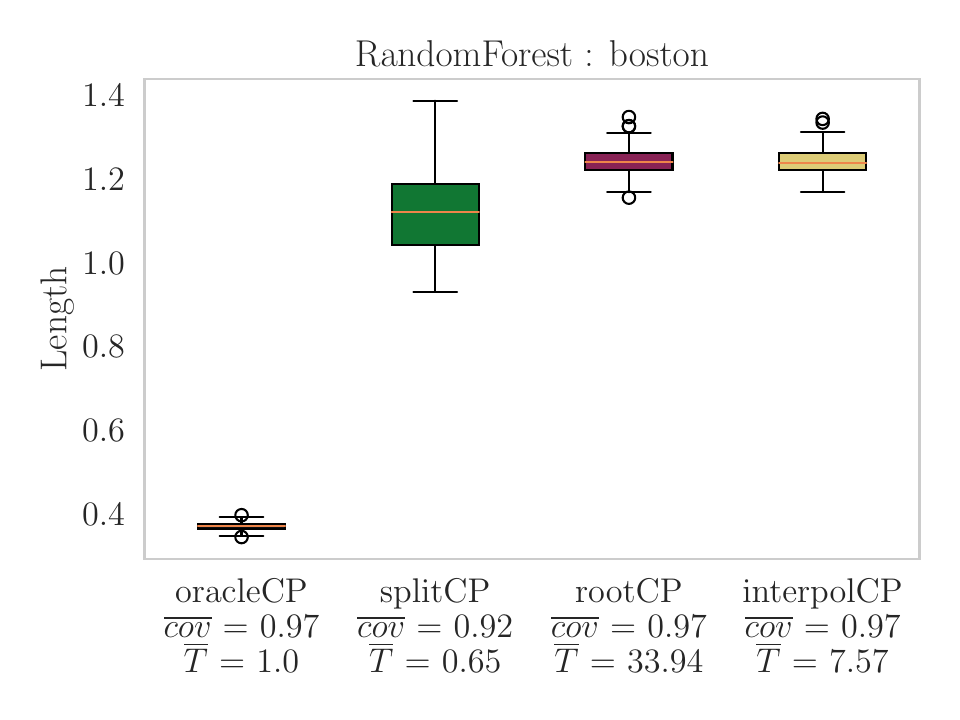}}
\subfigure{\includegraphics[width=0.49\textwidth]{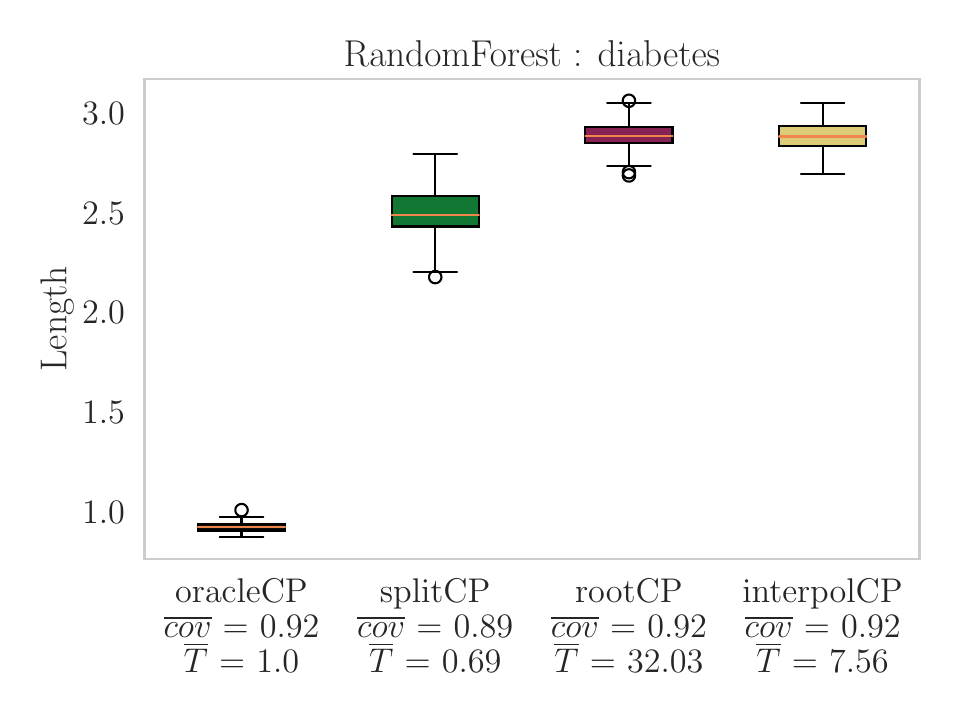}}
\subfigure{\includegraphics[width=0.49\textwidth]{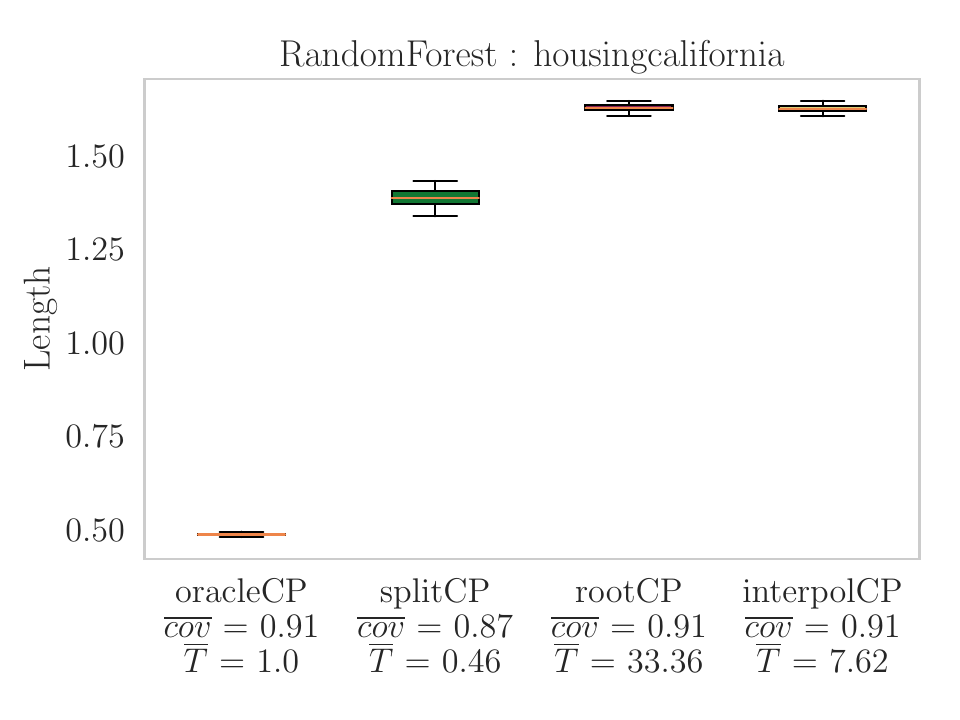}}
\subfigure{\includegraphics[width=0.49\textwidth]{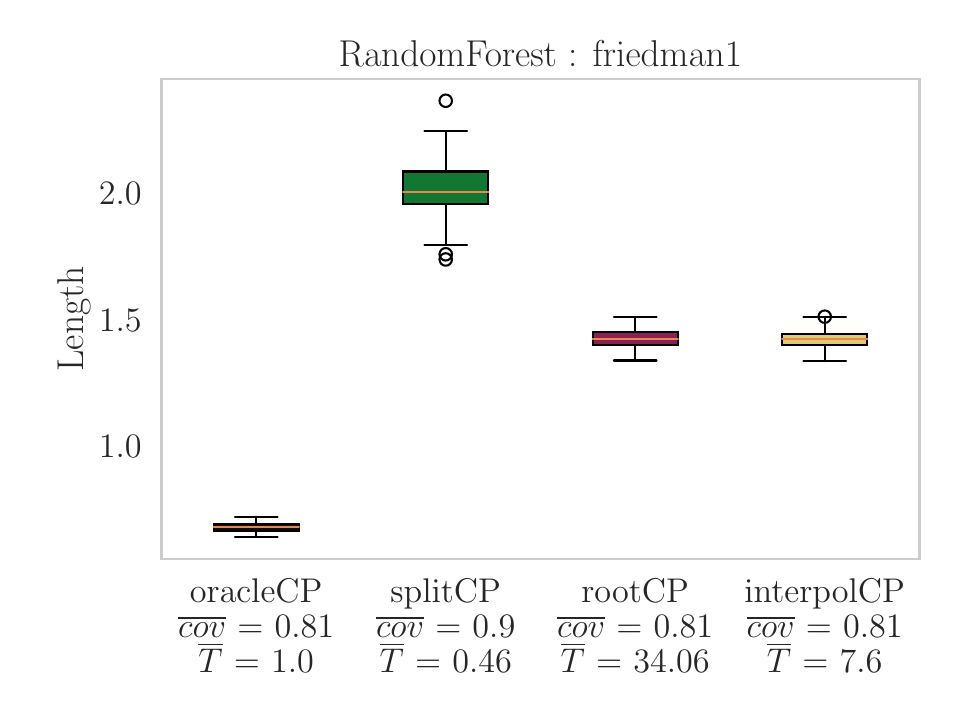}}
\caption{Benchmarking conformal sets for Random Forest regression models on real datasets. We display the lengths of the confidence sets over $100$ random permutation of the data. We denoted $\overline{cov}$ the average coverage, and $\overline{T}$ the average computational time normalized with the average time for computing \texttt{oracleCP} which requires a single model fit on the whole data. \label{fig:RandomForest_benchmarks} \looseness=-1}
\end{figure}
\begin{figure}
\subfigure{\includegraphics[width=0.49\textwidth]{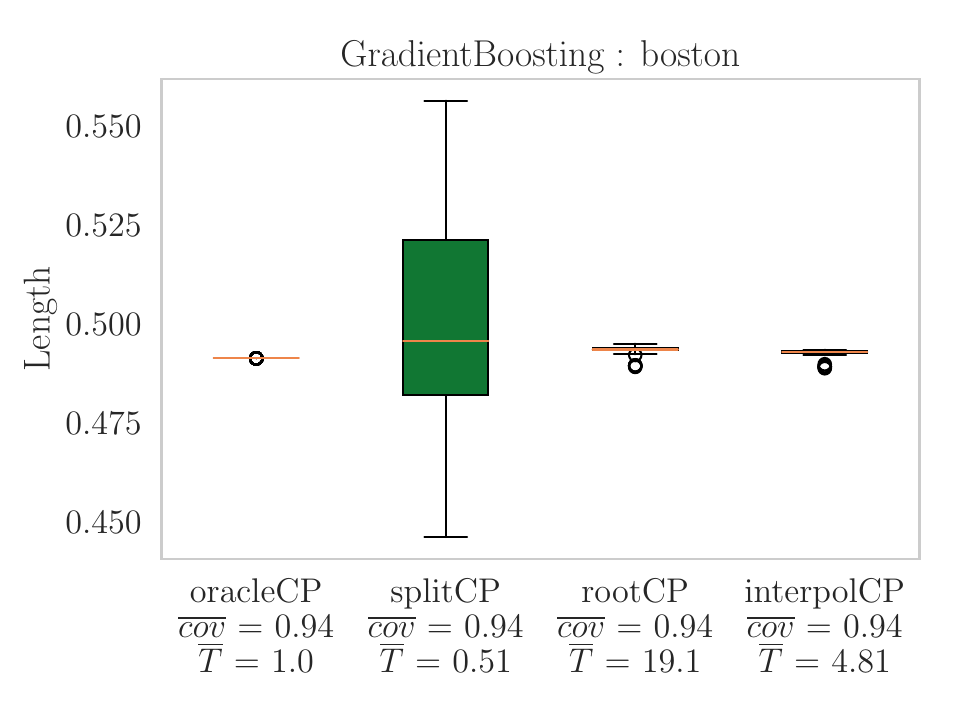}}
\subfigure{\includegraphics[width=0.49\textwidth]{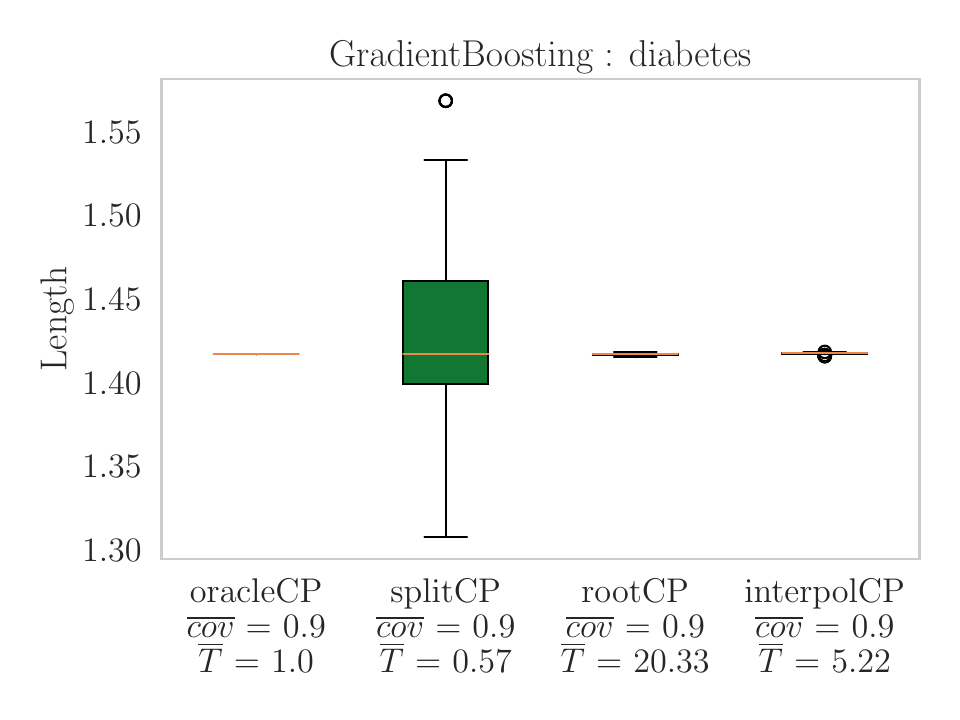}}
\subfigure{\includegraphics[width=0.49\textwidth]{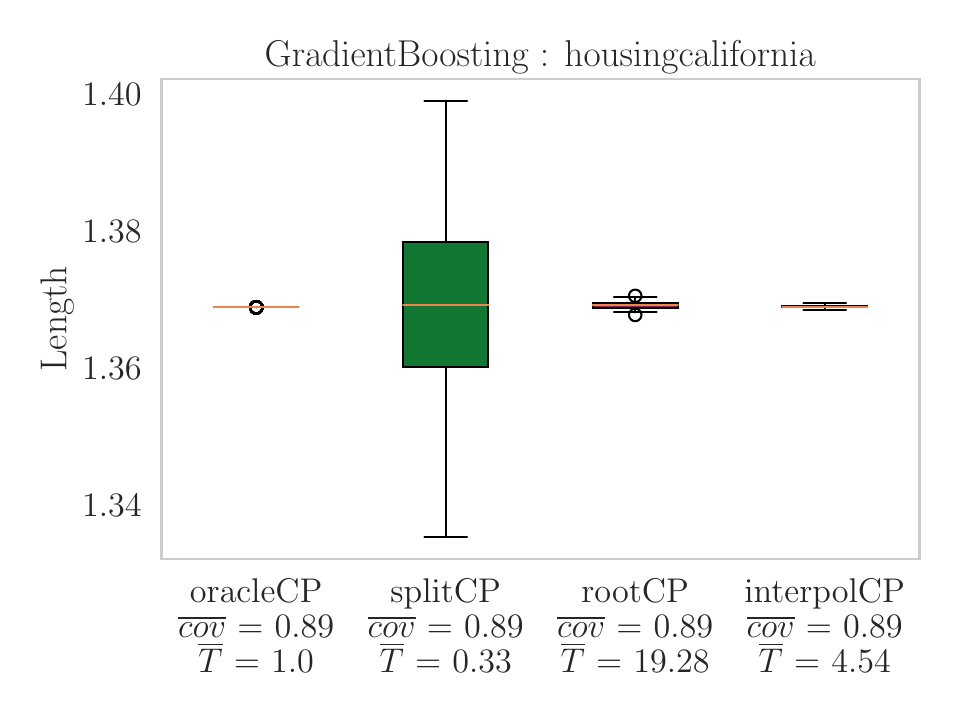}}
\subfigure{\includegraphics[width=0.49\textwidth]{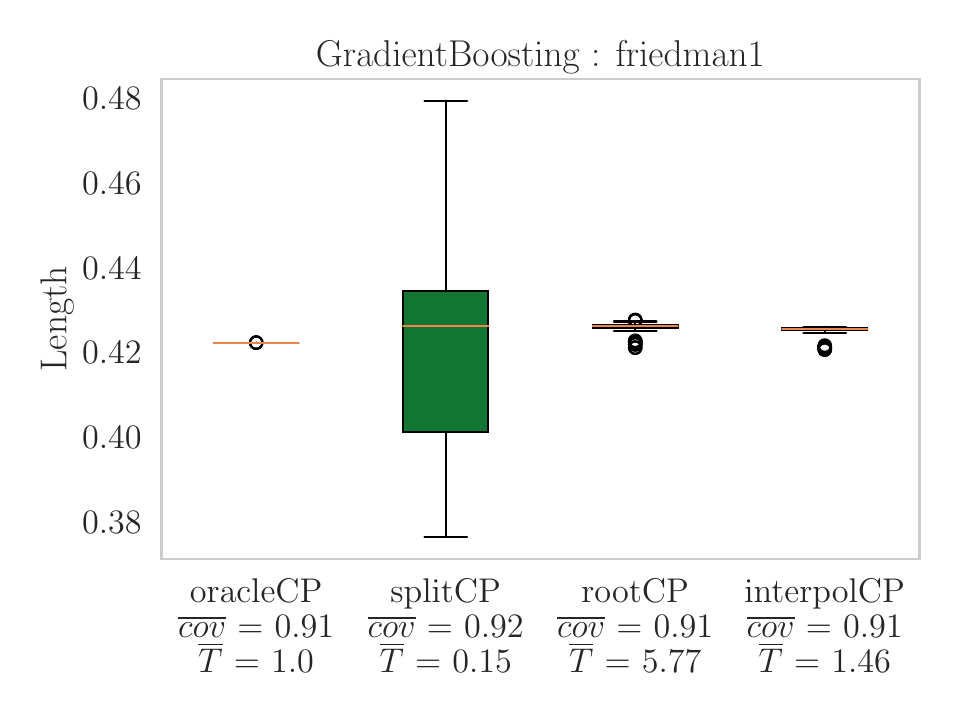}}
\caption{Benchmarking conformal sets for Gradient Boosting regression models on real datasets. We display the lengths of the confidence sets over $100$ random permutation of the data. We denoted $\overline{cov}$ the average coverage, and $\overline{T}$ the average computational time normalized with the average time for computing \texttt{oracleCP} which requires a single model fit on the whole data.\looseness=-1\label{fig:GradientBoosting_benchmarks}}
\end{figure}
 
We numerically examine the performance of the root-finding methods to compute various conformal prediction sets for regression problems on both synthetic, and real databases. \eug{We summarize the datasets in \Cref{tab:datasets}}.
 
\begin{table}[H]
\begin{center}
\begin{tabular}{cccccc}
\multicolumn{1}{l}{} & \multicolumn{1}{l}{Boston} & \multicolumn{1}{l}{Diabetes} & \multicolumn{1}{l}{Housingcalifornia} & \multicolumn{1}{l}{Friedman1} & \multicolumn{1}{l}{Climate} \\ \hline
$n$                  & $506$                      & $442$                        & $20640$                               & $500$                         & $814$                       \\ \hline
$p$                  & $13$                       & $10$                         & $8$                                   & $100$                         & $73570$                    
\end{tabular}
\caption{\eug{The first four datasets used in our experiments are available in \texttt{sklearn}. Here, $n$ is the number of samples, and $p$ is the number of features. The noise level in the Friedman dataset is fixed to $1$. The climate dataset is from NCEP/NCAR Reanalysis \citep{Kalnay_Kanamitsu_Kistler_Collins_Deaven_Gandin_Iredell_Saha_White_Woollen_Others96}}.\looseness=-1}
\label{tab:datasets}
\end{center}
\end{table}
 
The experiments were conducted with a coverage level of $0.9$\ie $\alpha = 0.1$.
For comparisons, we run the evaluations on $100$ repetitions of examples, and display the average of the following performance statistics for different methods: 1) the empirical coverage\ie the percentage of times the prediction set contains the held-out target $y_{n+1}$; 2) the length of the confidence intervals; 3) the execution time. For each run, we randomly select a couple of input/output $(x_i, y_i)$ to constitute the targeted variables for which we will compute the conformal prediction set, and the rest is considered as observed data $\Data_n$. Similar experimental setting was considered in \citep{Lei19}. \\
 
From \Cref{lm:empirical_repartition_lemma}, we have $\pi(y_{n+1}) \geq \alpha$ with probability larger than $1 - \alpha$. Whence one can define the \texttt{OracleCP} as $\pi^{-1}([\alpha, +\infty))$ where $\pi$ is obtained with a model fit optimized on the oracle data $\Data_{n+1}(y_{n+1})$. In the case where the conformity function is the absolute value, we obtain the reference prediction set as in \citep{Ndiaye_Takeuchi19}
$$\texttt{oracleCP: } [\mu_{y_{n+1}}(x_{n+1}) \;\pm\; Q_{1 - \alpha}(y_{n+1})] \enspace.$$
We remind that the target variable $y_{n+1}$ is not available in practice.

In the case of Ridge regression, \emph{exact} conformal prediction sets can be computed by homotopy without data splitting, and without additional assumptions \citep{Nouretdinov_Melluish_Vovk01}. This allows us to finely assess the precisions of the proposed approaches, and illustrate the speed up benefit in \Cref{fig:benchmark_ridge}.\\
 
\eug{We also illustrate the performance of our approach compared to the approximate homotopy method for the Lasso problem on a real data set from climate measurements. We first note that \texttt{splitCP} has a strictly, and significantly larger confidence set while the other approaches are quite close to the oracle performance. The approximate homotopy method uses all the data, and does not lose statistical efficiency if the model tolerance error is moderately small. However, as already noted in \citep{Ndiaye_Takeuchi19}, it becomes unusable when the accuracy of the optimization becomes low. This is because its complexity depends directly on the accuracy of the model optimization, see the discussion in \Cref{subsec:Approximation_to_a_Prescribed_Accuracy}. We have shown that this does not affect our method because the number of model fits does not increase as the optimization error decreases. In particular, we can observe in \Cref{tab:climate_lasso_bench} that \texttt{rootCP} is two to fifteen times faster when the tolerance errors range from $10^{-2}$ to $10^{-6}$. This is mainly because the complexity of the root finding approach, i.e., the number of times it calls the model, is independent of the optimization error of the underlying model fit. Whence, it allows the use of highly accurate estimators while maintaining feasible computations.}\\
 
\begin{table*}
\begin{center}
\begin{tabular}{llllll}
         &                   &                  & \multicolumn{3}{c}{\texttt{Approximate homotopy / rootCP}}                                          \\ \cline{4-6}
         & \texttt{oracleCP} & \texttt{splitCP} & \multicolumn{1}{c|}{1e-2}          & \multicolumn{1}{c|}{1e-4}           & \multicolumn{1}{c}{1e-6} \\
Coverage & 0.94              & 0.90             & \multicolumn{1}{l|}{0.95 / 0.93}   & \multicolumn{1}{l|}{0.94 / 0.94}    & 0.94 / 0.94              \\
Length   & 0.693             & 1.054            & \multicolumn{1}{l|}{0.747 / 0.699} & \multicolumn{1}{l|}{0.711 / 0.708}  & 0.707 / 0.707            \\
Time (s)     & 3.088             & 1.409            & \multicolumn{1}{l|}{4.289 / 2.478} & \multicolumn{1}{l|}{29.987 / 9.173} & 316.67 / 20.92          
\end{tabular}
\caption{Computing a conformal set for a Lasso regression problem on a climate data set with $n= 814$ observations, and $p= 73570$ features. We display the coverage, length, and execution times for different methods averaged over $100$ randomly held-out validation data sets. \label{tab:climate_lasso_bench}}
\end{center}
\end{table*}
 
We run experiments on more complex regression models such as ridge in \Cref{fig:ridge_benchmarks}, Lasso in \Cref{fig:Lasso_benchmarks}, Orthogonal Matching Pursuit (OMP) in \Cref{fig:OMP_benchmarks}, feedforward neural network also called Multi-Layer Perceptron (MLP) in \Cref{fig:MLP_benchmarks}, Random Forest in \Cref{fig:RandomForest_benchmarks}, and Gradient Boosting in \Cref{fig:GradientBoosting_benchmarks} (with warm start). In most of these settings, the estimator is obtained by approximating a solution of a non-convex optimization problem where none of the homotopy methods are available. We can observe in Figures~\ref{fig:ridge_benchmarks} to~\ref{fig:RandomForest_benchmarks} that the root-finding approach computes a full conformal prediction set while maintaining a reasonable computational time. In the worst cases observed, it costs about $30$ times a single model fit which roughly corresponds to the $15$ model fit for each root as predicted by the complexity (\egc $\epsilon=10^{-4}$). Moreover, this computational time is significantly reduced by the \texttt{interpolationCP} while achieving a statistical performance almost identical to the vanilla one in all our simulations. In all our experiments, we have chosen $d=8$ number of query points for the interpolation method.

\section{Conclusion}
Since its introduction, the computation of a confidence region with full conformal prediction methods has been a major weakness to its adoption by a broader audience. The algorithms available until now were based on too strong assumptions that limited them to estimators whose map $z \mapsto \mu_z(\cdot)$ can be traced\eg by using homotopy methods. We have shown that the limitations of the previous methods can be overcome by directly estimating the endpoints of the $\alpha$-level set of the typicalness function with a root finding algorithm. Therefore, it is unnecessary to train the regression estimator an infinite number of times nor to make strong additional assumptions on the prediction model. As long as the conformal set represents an interval containing the point prediction obtained from the observed data, it can easily be estimated with only about ten numbers of model fits. The proposed approach can be readily applied to recent generalizations of the conformal prediction set beyond the exchangeability assumption as in \citep{Chernozhukov_Wuthrich_Zhu18, Chernozhukov_Wuthrich_Zhu21}.\looseness=-1\\
 
Nevertheless, we insist that a full conformal prediction set is not always an interval. In this case, our approach fails without properly bracketing these intervals. Both, testing the interval assumption or finding a proper bracketing are still difficult. Another severe, and silent disadvantage is that the conformal set itself might be ill-defined when applied to regression estimators that depend on solving a non-convex optimization problem or a stochastic scheme (\egc stochastic gradient descent). Indeed, in these cases, given a fixed candidate $z$, $\pi(z)$ can take multiple values depending on the initialization or the random seed. Moreover, the symmetry assumptions might be violated if the instances are not used evenly (\egc when using stochastic gradient descent with importance sampling). As a future work, it would therefore be interesting to understand how these points can negatively affect the coverage guarantee, and the computational complexity.

\section*{Declarations}

\paragraph{Funding.}
IT was partially supported by MEXT KAKENHI (20H00601), JST CREST (JPMJCR21D3), JST Moonshot R\&D (JPMJMS2033-05), JST AIP Acceleration Research (JPMJCR21U2), NEDO (JPNP18002, JPNP20006), and RIKEN Center for Advanced Intelligence Project.

\paragraph{Conflicts of interest/Competing interests.}
The authors declare that they have no conflicts of interest.

\paragraph{Availability of data, and material.} All the data, and software used are available in open source.

\paragraph{Code availability.} The source of our implementation is available at 
\begin{center}
\url{https://github.com/EugeneNdiaye/rootCP}
\end{center}

\paragraph{Ethics approval.}
Not Applicable.

\paragraph{Consent to participate.}
All authors agreed to participate as authors of this article.

\paragraph{Consent for publication.} 
All authors consent for submission, and publication.

\paragraph{Authors' contributions.}
EN introduced, and formalized the contributions of the paper. He wrote, and analyzed the technical details, implemented the algorithms, designed, and performed the numerical simulations. IT contributed to the conceptualization, formulation, and design of the numerical experiments in this study.

\bibliography{references}

\begin{thebibliography}{}

\bibitem[\protect\citename{Allgower \& Georg, }2012]{Allgower_Georg12}
Allgower, E.~L., \& Georg, K. 2012.
\newblock {\em Numerical continuation methods: an introduction}.
\newblock Springer Science \& Business Media.

\bibitem[\protect\citename{Angelopoulos {\em et~al.\ }\relax,
  }2020]{Angelopoulos_Bates_Malik_Jordan20}
Angelopoulos, A., Bates, S., Malik, J., \& Jordan, M. 2020.
\newblock Uncertainty Sets for Image Classifiers using Conformal Prediction.
\newblock {\em ICLR}.

\bibitem[\protect\citename{Arlot \& Celisse, }2010]{Arlot_Celisse10}
Arlot, S., \& Celisse, A. 2010.
\newblock A survey of cross-validation procedures for model selection.
\newblock {\em Statistics surveys}.

\bibitem[\protect\citename{Bach {\em et~al.\ }\relax,
  }2012]{Bach_Jenatton_Mairal_Obozinski12}
Bach, F., Jenatton, R., Mairal, J., \& Obozinski, G. 2012.
\newblock Optimization with sparsity-inducing penalties.
\newblock {\em Foundations and Trends in Machine Learning}.

\bibitem[\protect\citename{Balasubramanian {\em et~al.\ }\relax,
  }2014]{Balasubramanian_Ho_Vovk14}
Balasubramanian, V., Ho, S-S., \& Vovk, V. 2014.
\newblock {\em Conformal prediction for reliable machine learning: theory,
  adaptations and applications}.
\newblock Elsevier.

\bibitem[\protect\citename{Barber {\em et~al.\ }\relax,
  }2021]{Barber_Candes_Ramdas_Tibshirani21}
Barber, R.~F., Candes, E.~J., Ramdas, A., \& Tibshirani, R.~J. 2021.
\newblock Predictive inference with the jackknife+.
\newblock {\em The Annals of Statistics}.

\bibitem[\protect\citename{Bates {\em et~al.\ }\relax,
  }2021a]{Bates_Angelopoulos_Lei_Malik_Jordan21}
Bates, S., Angelopoulos, A., Lei, L., Malik, J., \& Jordan, M. 2021a.
\newblock Distribution-Free, Risk-Controlling Prediction Sets.
\newblock {\em arXiv preprint arXiv:2101.02703}.

\bibitem[\protect\citename{Bates {\em et~al.\ }\relax,
  }2021b]{Bates_Candes_Lei_Romano_Sesia21}
Bates, S., Cand{\`e}s, E., Lei, L., Romano, Y., \& Sesia, M. 2021b.
\newblock Testing for Outliers with Conformal p-values.
\newblock {\em arXiv preprint arXiv:2104.08279}.

\bibitem[\protect\citename{Bousquet \& Bottou, }2008]{Bousquet_Bottou08}
Bousquet, O., \& Bottou, L. 2008.
\newblock The tradeoffs of large scale learning.
\newblock {\em NeurIPS}.

\bibitem[\protect\citename{Carlsson {\em et~al.\ }\relax,
  }2014]{Carlsson_Eklund_Norinder14}
Carlsson, L., Eklund, M., \& Norinder, U. 2014.
\newblock Aggregated conformal prediction.
\newblock {\em IFIP International Conference on Artificial Intelligence
  Applications and Innovations}.

\bibitem[\protect\citename{Cella \& Ryan, }2020]{Cella_Martin20}
Cella, L., \& Ryan, R. 2020.
\newblock Valid distribution-free inferential models for prediction.
\newblock {\em arXiv preprint arXiv:2001.09225}.

\bibitem[\protect\citename{Chang \& Hung, }2007]{Chang_Hung07}
Chang, Y-C., \& Hung, W-L. 2007.
\newblock LINEX loss functions with applications to determining the optimum
  process parameters.
\newblock {\em Quality \& Quantity}.

\bibitem[\protect\citename{Chen {\em et~al.\ }\relax,
  }2018]{Chen_Chun_Barber18}
Chen, W., Chun, K-J., \& Barber, R.~F. 2018.
\newblock Discretized conformal prediction for efficient distribution-free
  inference.
\newblock {\em Stat}.

\bibitem[\protect\citename{Chernozhukov {\em et~al.\ }\relax,
  }2018]{Chernozhukov_Wuthrich_Zhu18}
Chernozhukov, V., W{\"u}thrich, K., \& Zhu, Y. 2018.
\newblock Exact and robust conformal inference methods for predictive machine
  learning with dependent data.
\newblock {\em Conference On Learning Theory}.

\bibitem[\protect\citename{Chernozhukov {\em et~al.\ }\relax,
  }2021]{Chernozhukov_Wuthrich_Zhu21}
Chernozhukov, V., W{\"u}thrich, K., \& Zhu, Y. 2021.
\newblock An exact and robust conformal inference method for counterfactual and
  synthetic controls.
\newblock {\em Journal of the American Statistical Association}.

\bibitem[\protect\citename{Cox, }1975]{Cox75}
Cox, D.~R. 1975.
\newblock A note on data-splitting for the evaluation of significance levels.
\newblock {\em Biometrika}.

\bibitem[\protect\citename{Fisch {\em et~al.\ }\relax,
  }2021]{Fisch_Schuster_Jaakkola_Barzilay21}
Fisch, A., Schuster, T., Jaakkola, T., \& Barzilay, R. 2021.
\newblock Few-shot Conformal Prediction with Auxiliary Tasks.
\newblock {\em ICML}.

\bibitem[\protect\citename{Gammerman {\em et~al.\ }\relax,
  }1998]{Gammerman_Vovk_Vapnik98}
Gammerman, A., Vovk, V., \& Vapnik, V. 1998.
\newblock Learning by transduction.
\newblock {\em Proceedings of the Fourteenth conference on Uncertainty in
  artificial intelligence}.

\bibitem[\protect\citename{G{\"a}rtner {\em et~al.\ }\relax,
  }2012]{Gartner_Jaggi_Maria12}
G{\"a}rtner, B., Jaggi, M., \& Maria, C. 2012.
\newblock An exponential lower bound on the complexity of regularization paths.
\newblock {\em Journal of Computational Geometry}.

\bibitem[\protect\citename{Giesen {\em et~al.\ }\relax,
  }2010]{Giesen_Jaggi_Laue10}
Giesen, J., Jaggi, M., \& Laue, S. 2010.
\newblock Approximating parameterized convex optimization problems.
\newblock {\em European Symposium on Algorithms}.

\bibitem[\protect\citename{Gruber, }2010]{Gruber10}
Gruber, M. 2010.
\newblock {\em Regression estimators: A comparative study}.
\newblock JHU Press.

\bibitem[\protect\citename{Ho \& Wechsler, }2008]{Ho_Wechsler08}
Ho, S-S., \& Wechsler, H. 2008.
\newblock Query by transduction.
\newblock {\em IEEE transactions on pattern analysis and machine intelligence}.

\bibitem[\protect\citename{Hoerl, }1962]{Hoerl_62}
Hoerl, A.~E. 1962.
\newblock Application of ridge analysis to regression problems.
\newblock {\em Chemical Engineering Progress}.

\bibitem[\protect\citename{Hoerl \& Kennard, }1970]{Hoerl_Kennard70}
Hoerl, A.~E., \& Kennard, R.~W. 1970.
\newblock Ridge regression: Biased estimation for nonorthogonal problems.
\newblock {\em Technometrics}.

\bibitem[\protect\citename{Holland, }2020]{Holland20}
Holland, M.~J. 2020.
\newblock Making learning more transparent using conformalized performance
  prediction.
\newblock {\em arXiv preprint arXiv:2007.04486}.

\bibitem[\protect\citename{Kalnay {\em et~al.\ }\relax,
  }1996]{Kalnay_Kanamitsu_Kistler_Collins_Deaven_Gandin_Iredell_Saha_White_Woollen_Others96}
Kalnay, E., Kanamitsu, M., Kistler, R., Collins, W., Deaven, D., Gandin, L.,
  Iredell, M., Saha, S., White, G., Woollen, J., {\em et~al.\ }\relax. 1996.
\newblock The {NCEP/NCAR} 40-year reanalysis project.
\newblock {\em Bulletin of the American meteorological Society}.

\bibitem[\protect\citename{Kim {\em et~al.\ }\relax, }2020]{Kim_Xu_Barber20}
Kim, B., Xu, C., \& Barber, R. 2020.
\newblock Predictive inference is free with the jackknife+ after bootstrap.
\newblock {\em Advances in Neural Information Processing Systems}.

\bibitem[\protect\citename{Laxhammar \& Falkman, }2015]{Laxhammar_Falkman15}
Laxhammar, R., \& Falkman, G. 2015.
\newblock Inductive conformal anomaly detection for sequential detection of
  anomalous sub-trajectories.
\newblock {\em Annals of Mathematics and Artificial Intelligence}.

\bibitem[\protect\citename{Lehmann \& Romano, }2006]{Lehmann_Romano06}
Lehmann, E.~L., \& Romano, J.~P. 2006.
\newblock {\em Testing statistical hypotheses}.
\newblock Springer Science \& Business Media.

\bibitem[\protect\citename{Lei, }2019]{Lei19}
Lei, J. 2019.
\newblock Fast Exact Conformalization of Lasso using Piecewise Linear Homotopy.
\newblock {\em Biometrika}.

\bibitem[\protect\citename{Lei {\em et~al.\ }\relax,
  }2018]{Lei_GSell_Rinaldo_Tibshirani_Wasserman18}
Lei, J., G’Sell, M., Rinaldo, A., Tibshirani, R.~J., \& Wasserman, L. 2018.
\newblock Distribution-free predictive inference for regression.
\newblock {\em Journal of the American Statistical Association}.

\bibitem[\protect\citename{Linusson {\em et~al.\ }\relax,
  }2017]{Linusson_Norinder_Bostrom_Johansson_Lofstrom17}
Linusson, H., Norinder, U., Bostr{\"o}m, H., Johansson, U., \&
  L{\"o}fstr{\"o}m, T. 2017.
\newblock On the calibration of aggregated conformal predictors.
\newblock {\em Conformal and probabilistic prediction and applications}.

\bibitem[\protect\citename{Mairal \& Yu, }2012]{Mairal_Yu12}
Mairal, J., \& Yu, B. 2012.
\newblock Complexity Analysis of the Lasso Regularization Path.
\newblock {\em ICML}.

\bibitem[\protect\citename{Ndiaye \& Takeuchi, }2019]{Ndiaye_Takeuchi19}
Ndiaye, E., \& Takeuchi, I. 2019.
\newblock Computing Full Conformal Prediction Set with Approximate Homotopy.
\newblock {\em NeurIPS}.

\bibitem[\protect\citename{Ndiaye {\em et~al.\ }\relax,
  }2019]{Ndiaye_Le_Fercoq_Salmon_Takeuchi2019}
Ndiaye, E., Le, T., Fercoq, O., Salmon, J., \& Takeuchi, I. 2019.
\newblock Safe Grid Search with Optimal Complexity.
\newblock {\em ICML}.

\bibitem[\protect\citename{Nouretdinov {\em et~al.\ }\relax,
  }2001]{Nouretdinov_Melluish_Vovk01}
Nouretdinov, I., Melluish, T., \& Vovk, V. 2001.
\newblock Ridge regression confidence machine.
\newblock {\em ICML}.

\bibitem[\protect\citename{Obozinski \& Bach, }2016]{Obozinski_Bach16}
Obozinski, G., \& Bach, F. 2016.
\newblock A unified perspective on convex structured sparsity: Hierarchical,
  symmetric, submodular norms and beyond.
\newblock {\em HAL Id : hal-01412385, version 1}.

\bibitem[\protect\citename{Papadopoulos {\em et~al.\ }\relax,
  }2002]{Papadopoulos_Proedrou_Vovk_Gammerman02}
Papadopoulos, H., Proedrou, K., Vovk, V., \& Gammerman, A. 2002.
\newblock Inductive confidence machines for regression.
\newblock {\em European Conference on Machine Learning}.

\bibitem[\protect\citename{Pedregosa {\em et~al.\ }\relax,
  }2011]{Pedregosa_etal11}
Pedregosa, F., Varoquaux, G., Gramfort, A., Michel, V., Thirion, B., Grisel,
  O., Blondel, M., Prettenhofer, P., Weiss, R., Dubourg, V., Vanderplas, J.,
  Passos, A., Cournapeau, D., Brucher, M., Perrot, M., \& Duchesnay, E. 2011.
\newblock Scikit-learn: Machine Learning in {P}ython.
\newblock {\em J. Mach. Learn. Res}.

\bibitem[\protect\citename{Qin {\em et~al.\ }\relax, }2010]{Qin_Liu_Li10}
Qin, T., Liu, T-Y., \& Li, H. 2010.
\newblock A general approximation framework for direct optimization of
  information retrieval measures.
\newblock {\em Information retrieval}.

\bibitem[\protect\citename{Romano {\em et~al.\ }\relax,
  }2019]{Romano_Patterson_Candes19}
Romano, Y., Patterson, E., \& Candes, E.~J. 2019.
\newblock Conformalized quantile regression.
\newblock {\em NeurIPS}.

\bibitem[\protect\citename{Shafer \& Vovk, }2008]{Shafer_Vovk08}
Shafer, G., \& Vovk, V. 2008.
\newblock A tutorial on conformal prediction.
\newblock {\em Journal of Machine Learning Research}.

\bibitem[\protect\citename{Tibshirani, }1996]{Tibshirani96}
Tibshirani, R. 1996.
\newblock Regression shrinkage and selection via the lasso.
\newblock {\em Journal of the Royal Statistical Society. Series B
  (Methodological)}.

\bibitem[\protect\citename{Virtanen {\em et~al.\ }\relax, }2020]{Virtanen20}
Virtanen, P., Gommers, R., Oliphant, T.~E., Haberland, M., Reddy, T.,
  Cournapeau, D., Burovski, E., Peterson, P., Weckesser, W., Bright, J.,
  van~der Walt, S.~J., Brett, M., Wilson, J., Millman, K.~J., Mayorov, N.,
  Nelson, A. R.~J., Jones, E., K., R., E.~Larson, C J~Carey, Polat, İ., Feng,
  Y., Moore, E.~W., VanderPlas, J., Laxalde, D., Perktold, J., Cimrman, R.,
  Henriksen, I., Quintero, E.~A., Harris, C.~R., Archibald, A.~M., Ribeiro,
  A.~H., Pedregosa, F., van Mulbregt, P., \& {SciPy 1.0 Contributors}. 2020.
\newblock SciPy 1.0: fundamental algorithms for scientific computing in Python.
\newblock {\em Nature methods}.

\bibitem[\protect\citename{Vovk, }2015]{Vovk15}
Vovk, V. 2015.
\newblock Cross-conformal predictors.
\newblock {\em Annals of Mathematics and Artificial Intelligence}.

\bibitem[\protect\citename{Vovk {\em et~al.\ }\relax,
  }2005]{Vovk_Gammerman_Shafer05}
Vovk, V., Gammerman, A., \& Shafer, G. 2005.
\newblock {\em Algorithmic learning in a random world}.
\newblock Springer.

\bibitem[\protect\citename{Vovk {\em et~al.\ }\relax,
  }2017]{Vovk_Shen_Manokhin_Xie17}
Vovk, V., Shen, J., Manokhin, V., \& Xie, M. 2017.
\newblock Nonparametric predictive distributions based on conformal prediction.
\newblock {\em Conformal and Probabilistic Prediction and Applications}.

\bibitem[\protect\citename{Xie \& Huang, }2009]{Xie_Huang09}
Xie, H., \& Huang, J. 2009.
\newblock SCAD-penalized regression in high-dimensional partially linear
  models.
\newblock {\em The Annals of Statistics}.

\bibitem[\protect\citename{Xu \& Xie, }2021]{Xu_Xie20}
Xu, C., \& Xie, Y. 2021.
\newblock Conformal prediction interval for dynamic time-series.
\newblock {\em ICML}.

\end{thebibliography}
\bibliographystyle{authordate1}

\end{document}